\documentclass{article} 

\makeatletter
\usepackage{styles/fancyhdr}\@namedef{ver@fancyhdr.sty}{2013/01/28}
\usepackage{styles/natbib}\@namedef{ver@natbib.sty}{2010/09/13}
\makeatother

\usepackage{styles/iclr2027_conference}
\usepackage{times}

\usepackage{graphicx}
\usepackage{svg}

\usepackage{amsmath}
\usepackage{dsfont}\providecommand{\mathbbm}[1]{\mathds{#1}} 
\usepackage{amssymb}
\usepackage{amsthm}
\usepackage{xcolor}

\usepackage{soul} 
\sethlcolor{yellow}

\usepackage{bm}
\usepackage[disable]{todonotes}
\usepackage{subcaption}
\usepackage{booktabs}

\usepackage[utf8]{inputenc}
\usepackage[T1]{fontenc}
\usepackage{hyperref}
\usepackage{url}

\usepackage{nicefrac}
\usepackage{microtype}

\usepackage{tikz}
\usetikzlibrary{arrows.meta,positioning,calc}
\usepackage{styles/penn_colors}

\newtheorem{theorem}{Theorem}
\newtheorem{lemma}[theorem]{Lemma}
\newtheorem{corollary}[theorem]{Corollary}
\newtheorem{assumption}{Assumption}
\newtheorem{proposition}[theorem]{Proposition}
\newtheorem{definition}{Definition}
\newtheorem{remark}{Remark}

\input{styles/mysymbol.sty}
\title{Size-Transferability of Graph Transformers\\ with Convolutional Positional Encodings}

\author{%
  Javier Porras-Valenzuela\thanks{Corresponding author: \texttt{jporras@seas.upenn.edu}} \\
  Department of Electrical and Systems Engineering\\
  University of Pennsylvania\\
  Philadelphia, PA, USA \\
  \And
  Zhiyang Wang \\
  Hal{\i}c{\i}o\u{g}lu Data Science Institute\\
  University of California San Diego\\
  La Jolla, CA, USA \\
  \And
  Teresa Shang \\
  Department of Electrical and Systems Engineering\\
  University of Pennsylvania\\
  Philadelphia, PA, USA \\
  \AND
  Yusu Wang \\
  Hal{\i}c{\i}o\u{g}lu Data Science Institute\\
  University of California San Diego\\
  La Jolla, CA, USA \\
  \And
  Alejandro Ribeiro \\
  Department of Electrical and Systems Engineering\\
  University of Pennsylvania\\
  Philadelphia, PA, USA \\
}

\iclrfinalcopy 
\begin{document}

\maketitle
\lhead{Preprint.} 

\begin{abstract}
    Attention-based architectures have achieved remarkable success across data modalities, motivating the rise of Graph Transformers (GTs). GTs typically encode structural information via positional encodings (PEs), which may be either embedding vectors added to the input or modifications to the attention kernel. Given the challenges of learning from large graphs, a crucial question is whether GTs are \textit{size-transferable}, that is, whether they can generalize to graphs larger than those 
    seen during training. In this work, we study GTs through the lens of manifold limit models and establish that GTs can \textit{inherit} the transferability properties of their PEs. In particular, if the PEs are transferable, the GT is also transferable. We derive bounds showing that GTs with Graph Neural Network (GNN)-based PEs are transferable. Furthermore, our results show that their transferability is more stable than that of eigenvector-based GTs. We complement our theory with extensive experiments on standard graph benchmarks, where we confirm that GTs with GNN-based PEs exhibit superior transferability compared to those with eigenvector-based PEs. We showcase a real-world application of a computationally efficient GT in a shortest-path distance estimation task over terrain manifolds. Our results provide new insights into GTs and suggest practical directions for efficient training in large-scale settings.
\end{abstract}


\section{Introduction} \label{sec:intro}

Transformers have been adapted to graph-structured data by injecting adjacency information into node features via positional encodings (PEs) followed by attention, yielding Graph Transformers (GTs). They have become highly competitive across several domains, including large-scale molecular property prediction \citep{ying2021transformers}, biomedical knowledge graphs \citep{hu2020heterogeneous}, and long-range data benchmarks \citep{dwivedi2022long}. The key feature of GTs is performing pairwise attention between node features beyond their local connectivity, which improves their capacity to process long-range interactions.

However, this comes at the cost of increased computational complexity. Full pairwise attention over graphs scales quadratically with the number of nodes, which greatly increases the amount of data and computation required for training. Since this makes training directly on large-scale graphs prohibitively expensive, it is paramount to develop GTs that generalize to larger graphs than those seen during training, a property known as \textit{size-transferability}. 

Theoretical foundations pertaining to GT size-transferability are largely unexplored. In contrast, this key property is well-established for Graph Neural Networks (GNNs), which are composed of graph convolutional filters. Extensive theory from spectral graph signal processing establishes that graph filters and GNNs---under continuity conditions---are stable to perturbations and transferable across graphs sampled from the same limit model \citep{wang2024geometric, wang2024stability, ruiz2023transferability, ruiz2020graphon, keriven2020convergence}. These results are possible because graph convolutions incorporate graph structure by construction.

A central insight across the GT literature is that encoding graph structure explicitly using PEs enables effective attention-based architectures. Examples of PEs are embedding vectors added to node features, such as the top-$K$ eigenvectors of the graph Laplacian~\citep{dwivedi2020generalization}, which generalize the sinusoidal encodings of conventional Transformers. However, the transferability of eigenvector encodings depends on the spectral gap of the limit operator at the truncation point~\citep{dunson2021spectral}, a quantity that is not known in advance.

In this work, we present theoretical results that link the transferability of GTs to the transferability of their PEs. In particular, our main theoretical contribution is to show that GTs \textit{inherit} size-transferability from their PEs. We state this in our main theoretical contribution: 

\begin{itemize}
    \item [\textbf{[C1]}] We provide the first theoretical guarantees of GT transferability in Section~\ref{sec:transference}. Namely, we show that GTs converge to a Manifold Transformer, their limiting object in the manifold domain (Theorem~\ref{thm:gt_mt_convergence}). Leveraging this result, we show that if the PEs are transferable, size-transferability holds with a rate that depends on the encoding's transferability bound (Corollary~\ref{cor:gt_transferability}). Our theory extends readily to other Transformer variants.
\end{itemize}

Our results indicate that the choice of PE is a critical ingredient for deploying GTs across graphs of different sizes. This leads us to contrast the transferability properties of two different positional encodings through the lens of spectral theory and established results on GNN size transferability:

\begin{itemize}
    \item [\textbf{[C2]}] We compare the convergence of two popular encodings to their manifold limits: Laplacian eigenvector PEs and GNN-based PEs. We show how eigenvector PE convergence depends on the spectral gap (Theorem~\ref{thm:lap_pe_convergence}), which is unobservable in practice. In contrast, the convergence of GNN-based PEs depends on Lipschitz constants, which are controllable. 
\end{itemize}

Finally, we empirically verify the transferability properties of the architectures in two settings. 

\begin{itemize}
    \item [\textbf{[C3]}] In Section~\ref{sec:experiments}, we carry out experiments that illustrate transferable performance of GTs with eigenvector PEs and a GNN-based PE (RPEARL) on citation and social network datasets (arXiv-year, Reddit, SNAP-Patents, MAG). We empirically confirm that neighborhood-masked GTs, used for computational efficiency, also exhibit transferable performance. An extensive empirical verification of other popular GT variants and GNN PE backbones is provided in Appendix~\ref{app:other_pes}.

    \item [\textbf{[C4]}] As an application, we consider shortest path distance estimation over terrain graphs, using a transferable and computationally efficient GT.
\end{itemize}



\section{Related Work}
Limit models provide a foundation for analyzing GNN transferability, stability, and statistical generalization, including manifolds \citep{wang2024geometric,wang2024stability,wangManifoldPerspectiveStatistical2024}, graphons \citep{ruiz2020graphon,maskey2025generalization,maskey2023transferability}, and general limit objects \citep{levin2025transferring}. Transformer length generalization has been studied empirically and theoretically \citep{zhouTransformersCanAchieve2024,anilExploringLengthGeneralization2022,huangFormalFrameworkUnderstanding2024}, with positional encodings (PEs) playing an important role \citep{kazemnejadImpactPositionalEncoding2023,pengYaRNEfficientContext2023}. The study of GTs has been focused on positional encodings, including absolute spectral or learned PEs \citep{kreuzer2021rethinking,canturk2024gpse,kanatsoulisLearningEfficientPositional2024}, relative distance or diffusion encodings \citep{mialon2021graphit,ying2021transformers,maGraphInductiveBiases2023}, and sparse attention masks \citep{shirzad2023exphormer}. Related analyses address PE invariance, stability, and expressivity \citep{lim2023signbasis,huangstability,black2024comparing,stoll2025generalizable}. We compliment these studies by addressing a limit model framework to analyze the size transferability of graph transformers with an explicit connection to the properties of positional encodings.
In Section~\ref{sec_app:related_work} in the appendix, we discuss more related work.

\section{Graph Transformers with Convolutional Positional Encodings}
\label{sec:RPEAL}

\paragraph{Preliminary notation.}
An undirected graph $\bbG = (\ccalV, \ccalE, \ccalW)$ contains a node set $\ccalV$ with $N$ nodes and an edge set $\ccalE \subseteq \ccalV\times \ccalV$. The weight function $\ccalW: \ccalE \rightarrow \reals$ assigns values to the edges. We define the graph Laplacian $\bbL =\bbD - \bbA$, where $\bbA \in \reals^{N\times N}$ is the weighted adjacency matrix and the diagonal matrix $\bbD$ is called the degree matrix. Node signals are vectors denoted as  $\bbx_i \in \reals^q$ for each node $i\in[1,N]$. We collect node signals into a graph signal matrix of column vectors $\bbX = [\bbx_1,\dots,\bbx_N]$.

A GT is a layered architecture that processes a graph signal $\bbX$, and outputs another signal $\bbY \in \reals^{q\times N}$ via an attention~\citep{vaswaniAttention2017} operation, where each layer computes
\begin{align}
    \bbY &= \bbPhi_\bbG(\ccalT,\bbX) = \bbV\bbX\ \text{softmax}\left[(\bbQ\bbX)^\top(\bbK\bbX)\right]^\top.\label{eq:gt_attn_trf}
\end{align}
The network $\bbPhi_\bbG$ is parameterized by the query, key and value linear maps $\bbQ, \bbK, \bbV \in \reals^{q \times q}$, collected in $\ccalT = \{\bbQ,\bbK,\bbV\}$. 
The operation $(\bbQ\bbX)^\top(\bbK\bbX)$ computes inner products between pairs of projected node features $\langle\bbQ\bbx_i,\bbK\bbx_j\rangle$, for all $i,j \in [1,N]$. Then, the softmax operation normalizes the pairwise inner products into probability vectors along each row. Finally, the output node features are computed as a linear combination of the projected values weighted by the attention coefficients.\footnote{For simplicity of presentation, we omit the transformer's feedforward layer, softmax normalization constant, and low-rank projections. We present our theory on a single-layer, single-head transformer, but its extension to these settings is straightforward.}

Structural information is introduced by addition to the input $\bbX_0 = \bbX_{\text{in}} + \bbPsi_\bbG(\ccalH,\bbL,\bbZ)$, where $\bbPsi_\bbG: \reals^{p\times N} \rightarrow \reals^{q\times N}$ is called the graph \textit{positional encoder} and $\bbX_{\text{in}}$ is the graph input signal. The operator $\bbPsi_\bbG$ maps each node into a positional encoding vector utilizing the structural information of $\bbL$, and another graph signal $\bbZ$. Given that $\bbPsi_\bbG$ is the only piece of the architecture that depends on the graph structure, the choice of $\bbPsi_\bbG$ is crucial for effective graph learning. Popular examples include top-$k$ eigenvectors of $\bbL$~\citep{dwivedi2020generalization}, or learnable embeddings based on powers of the random walk matrix~\citep{dwivediGraphNeuralNetworks2021}. 
\todo{Consider pointing to an appendix comparing positional encodings.}




\subsection{Positional encodings with Graph Neural Networks} 
Consider a GT with a graph neural network (GNN) acting as a learnable positional encoder $\bbPsi_\bbG$.  A GNN is a layered architecture composed of graph convolutional filters followed by pointwise nonlinearities. One layer of a GNN is written as
\begin{align}
    \bbPsi_\bbG(\ccalH,\bbL,\bbZ) = \sigma \bigg[\sum_{k=0}^{K-1} \bbH_k\bbZ\bbL^k\bigg] = \sigma \bigg[ h(\bbL)\bbZ\bigg] \label{eqn:gnn_matrix_form_pe}
\end{align}

where $\bbZ \in \reals^{q\times N}$ stands for the input signals, $\ccalH = \{\bbH_0,\dots,\bbH_{K-1}\}$ collects the filter coefficients, $h(\bbL):\reals^{p\times N}\to\reals^{q\times N}$ denotes the graph filter in linear operator form. Equation~\eqref{eqn:gnn_matrix_form_pe} performs a graph convolution of order $K$ on the graph signals by the graph Laplacian $\bbL$, followed by the pointwise nonlinearity $\sigma$.

\begin{remark}
    Since GNNs are implicitly nonlinear functions of Laplacian eigenvectors (see Proposition 3.1~\citep{kanatsoulisLearningEfficientPositional2024}), they can be seen as an expressive alternative of Laplacian eigenvector-based positional encodings. 
\end{remark}

\section{Transferability Analysis of GTs via a Manifold Perspective} \label{sec:transference}

Transferability offers a principled way to overcome the practical limitations of quadratic attention by enabling models to capture structural regularities that persist across graph sizes and instances. In particular, if a sequence of graphs admits a well-defined limit object, then a model trained on small-scale graphs can, in principle, be transferred to larger-scale graphs without retraining, while preserving its functional behavior. In this section, we show that the transferability of GTs is inherited from the transferability of their PEs. That is, once PEs are size-transferable functions over the graph limit, the induced GT naturally generalizes across graph scales. Building on this insight, we present the first transferability analysis of GTs from a manifold-based perspective. 

\subsection{Discrete graphs and operator limits}
We study transferability by relating graph-based computations to their continuum counterparts on a manifold. Let $\ccalM$ be a $d$-dimensional compact, smooth and differentiable Riemannian submanifold embedded in a $\mathsf{M}$-dimensional space $\reals^\mathsf{M}$ with finite volume. This induces a measure $\mu$ which has a non-vanishing Lipschitz continuous density $\rho$ with respect to the Riemannian volume over the manifold with $\rho:\ccalM\rightarrow(0,\infty)$, assumed to be bounded as $0<\rho_{min}\leq \rho(x) \leq \rho_{max}<\infty$ for all $x\in\ccalM$. 
The manifold data supported on each point $x\in\ccalM$ is defined by vector-valued functions $g:\ccalM\rightarrow \reals^p$ \citep{wang2024stability}. We use $L^2(\mu; \reals^p)$ to denote square-integrable $\reals^p$-valued functions over $\ccalM$ with respect to measure $\mu$.  

Given a set of $N$ i.i.d. samples $X_N = \{x_i\}_{i=1}^N$ drawn from $\mu$ over $\ccalM$, we construct a graph $\bbG_N = \bbG(\ccalV,\ccalE,\ccalW)$ on these $N$ sampled points $X_N$, where each point $x_i$ is a vertex of graph $\bbG_N$, i.e. $\ccalV = X_N$. Each pair of vertices $(x_i,x_j)$ is connected with an edge while the weight attached to the edge $\ccalW(x_i,x_j)$ is determined by a kernel function $K_{\varepsilon_N}$, where $\varepsilon_N$ is the kernel bandwidth. The kernel function is decided by the Euclidean distance $\|x_i-x_j\|$. The graph Laplacian denoted as $\bbL_N$ can be calculated based on the weight function \citep{merris1995survey}. The constructed graph Laplacian with an appropriate kernel function has been proved to approximate the Laplace-Beltrami operator $\ccalL$ of $\ccalM$ \citep{calder2022improved, belkin2008towards, dunson2021spectral}. In our theory, we implement the normalized Gaussian kernel definition in \citep{dunson2021spectral}.

To compare discrete and continuous computations over the sampled graphs and over the manifold, we introduce two linear mappings that connect data from sampled graph $\bbG_N$ and manifold $\ccalM$ and back, i.e. a sampling operator as $\bbS_N : L^2(\mu;\reals^p) \rightarrow \reals^{p\times N}$,  and an interpolation operator as $\bbI_N : \reals^{p\times N} \rightarrow L^2(\mu;\reals^p)$. We assume that these operators are bounded and consistent \citep{wang2024geometric, levie2021transferability}, i.e. 
\begin{equation}
    \lim_{N\to \infty}\|\bbI_N \bbS_N g - g\|_{L^2(\mu)}=0, \forall\;g\in L^2(\mu;\mathbb R^p).
\end{equation}
With these mapping operators, the convergence of graph Laplacians $\bbL_N$ over constructed graphs $\bbG_N$ to the continuous operator $\ccalL$ can be written as
\begin{equation}
    \lim_{N\to \infty} \|\bbI_N \bbL_N \bbS_N g -  \ccalL g\|_{L^2(\mu)} = 0, \forall\;g\in L^2(\mu;\mathbb R^p).
\end{equation}

\subsection{Transferable Functions over Manifolds}
We can now define transferability as a property of functions defined through graph operators that remain well-defined in the continuum limit. Let $\bbPsi$ be a map defined on any Hilbert space with a bounded linear operator. We write this dependence explicitly as $\bbPsi_\ccalM(\ccalH,\ccalL, g)\in L^2(\mu; \reals^q)$ with $g\in L^2(\mu;\reals^p)$ as a manifold-level realization, and 
$\bbPsi_\bbG(\ccalH, \bbL_N, \bbS_N g) \in \reals^{q\times N}$ with $\bbS_N g \in \reals^{p \times N}$ as a graph-level realization. The same functional form $\bbPsi$ is thus instantiated on discrete graphs and on the continuum, with dependence on the underlying domain entering solely through the operator $\bbL_N$ or $\ccalL$. Both cases share the parametrization $\ccalH$.


\begin{definition}[Transferable functions over Manifolds]\label{def:transferability_pointwise} Assume that $\bbPsi$ is permutation equivariant and uniformly Lipschitz in its inputs.
    $\bbPsi$ is said to be transferable with respect to the operator convergence $\bbL_N \to \ccalL$ if, for every admissible input function $g\in L^2(\mu;\mathbb R^p)$, for every $x_i\in X_N$ in samples of $\ccalM$, it holds with probability at least $1-\delta_N$,
\begin{equation}\label{eqn:transferability_pointwise}
   \big\| \bm\Psi_\bbG(\ccalH, \bbL_N, \bbS_N g)(x_i) - \bm\Psi_\ccalM(\ccalH, \mathcal L, g)(x_i) \big\| \leq \Delta_{\bbPsi(\bbG)},
\end{equation}
where $\Delta_{\bbPsi(\bbG)} = \Delta_{\bbPsi}(N)$ and $\delta_N$ are sequences satisfying $\Delta_{\bbPsi}(N)\to 0$ and $\delta_N\to 0$ as $N\to\infty$.
\end{definition}

Definition~\ref{def:transferability_pointwise} formalizes the notion of transferability as a function whose output converges to the limit object operating on the manifold. As a consequence, a transferable function $\bm\Psi_{\bbG}$ can be trained or defined on finite graphs of one size and deployed on other graph instances---possibly of different sizes or resolutions---as long as the associated operators converge to the same manifold limit $\ccalL$.

\subsection{Transferable Graph Transformers}
Under the transferability framework established above, we now turn to characterizing transferability for a Graph Transformer to its counterpart limit, a Manifold Transformer.

\begin{definition}[Softmax Graph Transformer] \label{def:softmax_gt}
    A Graph Transformer (GT) layer is given by
    \begin{align}
        \bbPhi_\bbG(\ccalT,\bbX)_i = \frac{\sum_{j=1}^N \exp\langle \bbQ \bbx_i, \bbK \bbx_j\rangle\, \bbV \bbx_j}{\sum_{j=1}^N \exp\langle \bbQ \bbx_i, \bbK \bbx_j\rangle}, \label{eqn:softmax_gt_attn}
    \end{align}
where $\bbQ$, $\bbK$, and $\bbV$ are the query, key, and value matrices,  respectively.
The input to the GT is the output of a positional encoding module, $\bbX = \bbPsi_\bbG(\ccalH,\bbL_N,\bbS_N g) \in \reals^{q\times N}$.
\end{definition}
\begin{definition}[Softmax Manifold Transformer]\label{def:softmax_mt}
    A Manifold Transformer (MT) layer is given by
    \begin{align}
        \bm\Phi_\ccalM(\ccalT,f)(x) = \frac{\int_\ccalM \exp\langle \bbQ f(x), \bbK f(y)\rangle\, \bbV f(y)\, d\mu(y)}{\int_\ccalM \exp\langle \bbQ f(x), \bbK f(y)\rangle\, d\mu(y)}, \label{eqn:softmax_mt_attn}
    \end{align}
    for all $x\in \ccalM$, with the input $f \in L^2(\mu;\reals^q)$ defined as the positional encoding of a manifold signal $g \in L^2(\mu;\reals^p)$, that is, $f(x) = \bm\Psi_\ccalM(\ccalH,\ccalL,g)(x)$, where $\bm\Psi_\ccalM(\ccalH,\ccalL,g):\ccalM \rightarrow \reals^q$ is the manifold positional encoding.
\end{definition}

The MT is a map $\bm\Phi_\ccalM: L^2(\mu;\reals^p) \rightarrow L^2(\mu;\reals^q)$ resulting from the composition of $f$ (the positional encoding) with manifold attention. The vector $f(x)$ corresponds to the output of the positional encodings in the manifold domain. The manifold attention of Equation~\eqref{eqn:softmax_mt_attn} is the continuous analogue of Equation~\eqref{eqn:softmax_gt_attn}, with infinite attention coefficients for each point in the manifold. In a similar fashion, the manifold positional encoding $\bbPsi_\ccalM$ is the limit object of the graph positional encoding $\bbPsi_\bbG$. 

The following set of assumptions are required to ensure the convergence of GTs to MTs.

\begin{assumption}[Bounded linear operators] \label{assm:linear_ops}
 $\bbQ$, $\bbK$, and $\bbV$ are bounded linear operators with constants $C_Q, C_K, C_V > 0$, i.e., $\|\bbQ \bbx\| \le C_Q \|\bbx\|$, $\|\bbK \bbx\| \le C_K \|\bbx\|$, $\|\bbV \bbx\| \le C_V \|\bbx\|$, $\forall\, \bbx \in \reals^q$. We write $C_\ccalT = C_QC_KC_V$.
\end{assumption}

\begin{assumption}[Bounded and normalized Lipschitz positional encoding]\label{assm:normalized_lipschitz_pe}
For all $x,y \in \ccalM$, the positional encoding function satisfies
$\|\bbPsi_\ccalM(\ccalH,\ccalL,g)(x) - \bbPsi_\ccalM(\ccalH,\ccalL,g)(y)\|_2 \leq d_\ccalM(x,y)$,
where $d_\ccalM$ is the geodesic distance on $\ccalM$, and is bounded as
$\|\bbPsi_\ccalM(\ccalH,\ccalL,g)(x)\|_2 \leq 1$. The graph positional encoding is also normalized with $\|\bbPsi_\bbG(\ccalH,\bbL_N,\bbS_N g)(x_i)\|_2 \leq 1$ for all $x_i\in X_N$. 
\end{assumption}

Assumption \ref{assm:normalized_lipschitz_pe} is a mild assumption on the PE limit object. Assumption \ref{assm:linear_ops} can be seen as having regularized transformer operators. These can be realized in practice via penalty terms during the training process.
We can now present our main theorem on the convergence of GT to MT, which we will use to show the transferability of GT.

\begin{theorem} (Point-wise Convergence of GT to MT) \label{thm:gt_mt_convergence}
    Assume $\bbPsi_\bbG$ is a transferable positional encoding with limit object $\bbPsi_\ccalM$, i.e., $\|\bbPsi_\bbG(\ccalH,\bbL_N,\bbS_N g)(x_i) - \bbPsi_\ccalM(\ccalH,\ccalL,g)(x_i)\| \leq \Delta_{\bbPsi(\bbG)}$ for each $x_i\in X_N$, where $\ccalH$ collects learnable parameters. Under Assumptions \ref{assm:linear_ops} and \ref{assm:normalized_lipschitz_pe}, for each node $x_i\in X_N$, the output difference between GT and MT, with probability at least $1-\delta$, is bounded by
    \begin{align}
        \| \bm\Phi_\bbG(\ccalT,\bbX)(x_i) - \bm\Phi_\ccalM(\ccalT,f)(x_i) \|_2 \nonumber&\leq          \Delta_{\bbPhi(\bbG)} \\
        &= \big[C_V + 4e^{2B}C_\ccalT\big]\Delta_{\bbPsi(\bbG)} + \big[ e^B C_V + 2 e^{2B} C_\ccalT \big] r,
    \end{align}
    where $r$ is the radius of the balls containing the partitions of $\ccalM$ induced by $X_N$ (Appendix~\ref{app:induce}), which satisfies $r \leq A\left({\log N}/{N}\right)^{1/d}$ if $d\geq 3$ and $r \leq A(\log N)^{3/4} / N^{1/2}$ if $d=2$, with $A$ a constant related to the geometry of $\ccalM$, $d$ the intrinsic dimension of the manifold, and $B := C_QC_K$.
\end{theorem}

The proof of Theorem \ref{thm:gt_mt_convergence} is available in Appendix \ref{app:gt_mt_cvg_proof}.
This result proves that as the number of nodes sampled in graph $\bbG$ increases, the output of GT tends to converge to the underlying MT, with the discretization term decaying at a rate $\ccalO(({\ \log N/N})^{1/d})$ and the remaining term decaying with the PE error $\Delta_{\bbPsi(\bbG)}$.
The dependence of this result on the linear operator bounds $C_\ccalT$ and $C_V$, as well as the bound on the positional encoding $\Delta_{\bbPsi(\bbG)}$, suggests that smoother linear operators on the transformer and positional encoding can improve the convergence rate. 

Theorem \ref{thm:gt_mt_convergence} indicates that if the PEs are transferable, the pointwise output difference of GT and MT decays as the size of the sampled graph $N$ grows. This implies that we can train a GT on a small graph $\bbG_1$ with $N_1$ nodes, and \textit{transfer} it to a larger graph $\bbG_2$, with $N_2 > N_1$, with guarantees that the approximation gap to the MT's output is bounded. This implication is meaningful in practice given the $\ccalO(N^2)$ computational cost of GT -- we can train on a relatively smaller graph and ensure good performance on larger graphs. We state this explicitly in the following corollary:

\begin{corollary} (Transferability of GTs) \label{cor:gt_transferability}
    Let $\bbG_1$ and $\bbG_2$ be graphs constructed by points sampled from $\ccalM$, with $N_1$, $N_2$ nodes respectively, and graph signals $\bbX_1$ and $\bbX_2$, and let $\|\cdot\|_{L^{1,2}(\ccalM)}$ be the norm defined in Appendix~\ref{app:induce}.  Then, it holds, with probability $1-\delta$,
    \begin{align}
   \nonumber      \frac{1}{\mu(\ccalM)} &|| \bbI_{N_1}\bm\Phi_{\bbG_1}(\ccalT,\bbX_1) - \bbI_{N_2}\bm\Phi_{\bbG_2}(\ccalT,\bbX_2)||_{L^{1,2}(\ccalM)} \leq \Delta_{\bbPhi(\bbG_1)}  + \Delta_{\bbPhi(\bbG_2)} \nonumber \\
   &\qquad\qquad\qquad\qquad\qquad\qquad\qquad\qquad + 4e^{2B}C_\ccalT \left( r_1 + r_2\right),
    \end{align}
    where $r_1$ and $r_2$ are the partition radii of $\bbG_1$ and $\bbG_2$ (Appendix~\ref{app:induce}), bounded as in Theorem~\ref{thm:gt_mt_convergence}.
\end{corollary}
\begin{remark}
    Theorem~\ref{thm:gt_mt_convergence} and Corollary~\ref{cor:gt_transferability} are proven for a generalized version of the softmax GT presented in this section. The same proof strategy can be extended to other attention mechanism beyond softmax when the discrete operators could approximate a well-defined continuum operator. Furthermore, the framework extends to point clouds and meshes whose associated discrete operators and sampling schemes satisfy the assumed manifold convergence conditions.
\end{remark}


\section{Convergence of Positional Encodings} \label{sec:pe_convergence}
Corollary~\ref{cor:gt_transferability} expresses a notion of transferability \textit{inheritance} ---  if the positional encodings $\bbPsi_\bbG$ are transferable, then the following GT $\bbPhi_\bbG$ transfers as well. The choice of $\bbPsi_\bbG$ is the primary factor we can control in improving GT transferability via the PE convergence term $\Delta_{\bbPsi(\bbG)}$. In this section we compare the convergence bounds for two choices of $\bbPsi_\bbG$.

\subsection{Eigenvector Positional Encoding Convergence} \label{sec:lap_pe_convergence}
A common GT positional encoding consists of truncations of the graph Laplacian eigenspace~\citep{dwivedi2020generalization}. The analogue in the manifold domain is to truncate to a fixed set of eigenfunctions of the Laplace--Beltrami operator. Let $\{(\lambda_j,\bm\phi_j)\}_{j\geq 0}$ and $\{(\lambda_{j,N},\bbphi_{j,N})\}_{j=0}^{N-1}$ be the eigenpairs of $\ccalL$ and $\bbL_N$, i.e., $\ccalL\bm\phi_j=\lambda_j\bm\phi_j$ and $\bbL_N\bbphi_{j,N}=\lambda_{j,N}\bbphi_{j,N}$, with eigenvalues sorted in increasing order. The graph and manifold PEs can be defined as 
\begin{definition}[Eigenvector and Eigenfunction positional encodings]
    \begin{align}
        \bbPsi_\bbG^{(K)}(\bbL_N)(x_i)
        &:=\begin{bmatrix}[\bbphi_{1,N}]_i&\cdots&[\bbphi_{K,N}]_i\end{bmatrix}^{\top},\label{eq:laplacian_pe_def}\\
        \bbPsi_\ccalM^{(K)}(\ccalL)(x)
        &:=\begin{bmatrix}\bm\phi_1(x)&\cdots&\bm\phi_K(x)\end{bmatrix}^{\top}.
    \end{align}
    Furthermore, denote their eigengaps by $\Gamma_{K,N}$ and $\Gamma_{K}$, respectively.
\end{definition}
The graph eigenvectors are an approximation of the Laplace--Beltrami eigenfunctions. According to the Davis--Kahan theorem, this approximation is inversely proportional to the eigengap. This leads to unstable transferability, as we can see in the following theorem.

\begin{theorem}[Convergence of truncated Laplacian eigenvectors to manifold limit]\label{thm:lap_pe_convergence}
    Fix $K>0$ and let $\varepsilon_N\geq\left(\log N/N\right)^{1/(2d+8)}$. Then, there exist signs $r_{j,N}\in\{-1,1\}$, $1\leq j\leq K$, such that, with $\bbR_N:=\mathrm{diag}(r_{1,N},\ldots,r_{K,N})$, for every $x_i\in X_N$, with probability at least $1-N^{-2}$,
    \begin{align}
        &\left\|\bbR_N\bbPsi_\bbG^{(K)}(\bbL_N)(x_i)
        -\bbPsi_\ccalM^{(K)}(\ccalL)(x_i)\right\|_2 \leq\Delta_{\bbPsi(\bbG_N)}^{\mathrm{Lap}}(K)
        =\frac{\sqrt{K}\,\Omega_K}{\Gamma_K}\,\varepsilon_N^2,
    \end{align}
    where $\Gamma_K:=\min_{1\leq j\leq K}(\lambda_{j+1}-\lambda_j)>0$ is the eigengap of $\ccalL$, and the constant $\Omega_K$ depends on $\lambda_K$, $d$, and the diameter, volume, injectivity radius, curvature and second fundamental form of $\ccalM$. 
\end{theorem}

The proof is in Appendix~\ref{app:lap_pe_cvg_proof}. The convergence rate is inversely proportional to the eigengap of the Laplace--Beltrami operator, $\Gamma_K$. Consequently, for eigenvalues that are close, convergence is more unstable. In principle, the instability could be addressed by picking a value of $K$ corresponding to large $\Gamma_K$. However, $\Gamma_K$ is not available in practice, and so $K$ is chosen based on $\Gamma_{K,N}$, the eigengap of the graph Laplacian, which requires a sample of at least order $\left(\log N/N\right)^{1/(2d+8)}$ for an accurate estimate. See Appendix~\ref{app:eigengaps} for more discussion on eigengap approximations.

Finally, the sign matrix $\bbR_N$ aligns the signs of eigenvectors, given that $\bbphi_{j,N}$ and $-\bbphi_{j,N}$ are both eigenvectors of $\bbL_N$. This sign ambiguity poses a practical difficulty for PE methods that explicitly compute eigendecompositions. It is typically addressed by techniques such as augmenting training data with random eigenvector flips~\citep{dwivedi2020generalization} or smoothening eigenvalues with a permutation equivariant function~\citep{huangstability}.

\subsection{GNN Positional Encoding Convergence} \label{sec:gnn_pe_convergence}
The challenges of eigenvector positional encodings, namely, their unstable convergence and sign ambiguity, can be avoided by using a smooth positional encoding that implicitly operates on eigenvectors, such as GNNs. We consider GNNs, whose size-transferability is well-established, and whose limit object is a manifold neural network (MNN), which can be defined as 

\begin{definition}[Manifold neural network]\label{def:mnn}
    One layer of an MNN maps the manifold signal $g$ to
    \begin{equation}\label{eqn:mnn_main}
        \bbPsi_\ccalM(\ccalH,\ccalL,g)(x)=\sigma\bigg[ h(\ccalL) g(x)\bigg],
    \end{equation}
    where $h(\ccalL)$ applies the filter determined by the coefficients
    $\ccalH=\{h_0,\dots,h_{K-1}\}\subset\reals$ based on Laplace-Beltrami operator $\ccalL$, and $\sigma$ is a
    pointwise nonlinearity.
\end{definition}
%
We now present the convergence of GNNs to MNNs as follows.

\begin{theorem}[Point-wise Convergence of GNN to MNN]\label{thm:gnn_mnn_convergence}
    Let $\bbPsi_\bbG$ and $\bbPsi_\ccalM$ be a GNN and MNN with the same filters $\ccalH$. Under the normalized Lipschitz signal and nonlinearity assumptions, and the non-amplifying and spectral continuity assumptions on the filters, for every node $x_i\in X_N$, with probability at least $1-\delta$,
    \begin{align}
        &\left\|\bbPsi_\bbG(\ccalH,\bbL_N,\bbS_N g)(x_i)
        -\bbPsi_\ccalM(\ccalH,\ccalL,g)(x_i)\right\|_2\leq\Delta_{\bbPsi(\bbG_N)}^{\mathrm{GNN}}
        =C\varepsilon_N^2+\frac{\pi^2}{6}\sqrt{\frac{\log(1/\delta)}{N}},
    \end{align}
    where $C$ depends on the geometry of the manifold and scales with the filter constant $C_L$.
\end{theorem}

The GNN bound of Theorem~\ref{thm:gnn_mnn_convergence} does not depend on the truncated eigengap $\Gamma_K$. Instead we have the GNN Lipschitz constant $C_L$, which can be controlled via regularization. 
Therefore, GNNs are more stable than eigenvectors in the sense of size-transferability. In contrast to the eigenvector convergence bound of Theorem~\ref{thm:lap_pe_convergence}, the GNN bound does not depend on the spectral gap. This result is adapted to our setting from the GNN to MNN convergence results in~\citet{wangManifoldPerspectiveStatistical2024}.

Finally, in Appendix~\ref{app:composed_corollaries} we present Corollaries~\ref{cor:lap_transferability} and~\ref{cor:gnn_transferability} to state the explicit transferability on GTs with eigenvector and GNN PEs, followed by application of Theorems~\ref{thm:lap_pe_convergence} and~\ref{thm:gnn_mnn_convergence} into Corollary~\ref{cor:gt_transferability}.

\begin{figure*}[h]
    \centering
    \includegraphics[width=\linewidth]{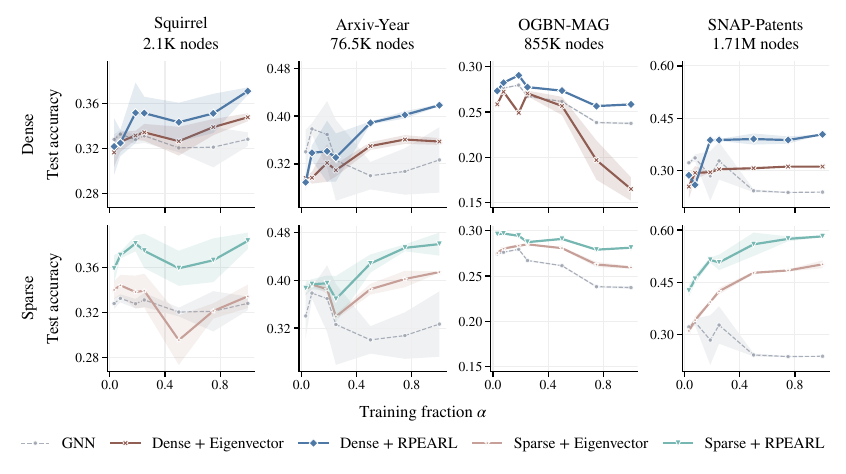}
    \caption{Transferability of GTs with RPEARL and Laplacian PEs measured by accuracy (higher is better $\uparrow$). The $x$-axis represents the train graph sizes as a proportion of the largest graph $(\alpha)$, and the $y$-axis is the test accuracy at the full-sized graph. The top row shows dense GTs, and the bottom row shows sparse GTs. The GNN baseline appears in both rows for reference. Shaded regions show the standard deviation over 3 runs. The titles show dataset name and largest graph size.}
    \label{fig:transferability_main}
\end{figure*}

\section{Experiments} \label{sec:experiments}
\subsection{Transferability on graph datasets} \label{sec:experiments_transferability_graph_ds}

Corollary~\ref{cor:gt_transferability} predicts a shrinking gap between GTs and MTs as graph sizes grows. In Figure \ref{fig:transferability_main}, from the full graph, we first subsample a large test graph $\bbG_{\text{TST}}$ with $N_{\text{TST}}$ nodes. We then build training graphs by randomly sampling $N_{\text{TR}} = \alpha N_{\text{TST}}$ nodes from $\bbG_{\text{TST}}$ and taking the induced subgraphs, for training fractions $\alpha \in \{0.03, 0.08, 0.18, 0.25, 0.50, 0.75, 1.0\}$. We train each GT on its subgraph and evaluate on $\bbG_{\text{TST}}$, expecting the performance to improve  as $\alpha$ increases and gradually approach that of a GT trained on the full test graph ($\alpha = 1$). In Figure \ref{fig:heatmap_transferability}, we subsample both training and testing graphs with different ratios, to further inspect the relationship between training sizes, testing sizes and accuracy.

\paragraph{Datasets.} We evaluate node classification on Squirrel \cite{rozemberczki2021multi}, ArXiv-year~\citep{limLargeScaleLearning2021}, OGBN-MAG~\citep{wangMicrosoftAcademicGraph2020}, and SNAP-Patents \citep{limLargeScaleLearning2021}. Additional graph classification results on Reddit-Binary~\citep{yanardagDeepGraphKernels2015}, appear in Appendix \ref{app:additional_results}.

\paragraph{Models.}
We compare dense GTs (Definition~\ref{def:softmax_gt}) and sparse GTs with $k$-hop attention masks (Appendix~\ref{app:sparse_gt}), using RPEARL convolutional PEs \eqref{eq:rpearl_pe} or Laplacian eigenvector PEs \eqref{eq:laplacian_pe_def}. Appendix~\ref{app:other_pes} includes additional PEs (Random Walk PEs and other GNNs), and backbones (Exphormer \citep{shirzad2023exphormer} and GraphGPS \citep{rampavsek2022recipe}).


\paragraph{Convolutional PEs yield more stable transferability than eigenvector PEs.}
Our empirical results are consistent with Theorems~\ref{thm:lap_pe_convergence} and~\ref{thm:gnn_mnn_convergence} and show that GTs inherit the more stable convergence of GNN-based PEs. Figure \ref{fig:transferability_main} shows that performance with RPEARL PEs improves as the training graph grows and then plateaus, while performance with Laplacian PEs can be unstable and degrade. On OGBN-MAG, the accuracy of the dense GT with Laplacian PEs peaks near $\alpha = 0.25$ and then drops sharply as the training graph grows. We do not observe this degradation with RPEARL PEs. Figure \ref{fig:heatmap_transferability} shows accuracy across train and test graph sizes is robust. We observe that with the same training fraction, RPEARL PEs achieve better accuracy when transferred to larger testing graphs both in dense and sparse GT settings compared to Laplacian PE. Full observation of heatmap results can be found in Figure \ref{fig:heatmaps_sparse_dense_and_sparse_gt} in Appendix \ref{app:additional_results}. This indicates that GTs with RPEARL trained on smaller graphs can transfer effectively to larger graphs without retraining, supporting the use of smaller training graphs to reduce computational cost.

\paragraph{Sparse GTs preserve transferability at a fraction of the cost.}
Empirically, sparse GTs transfer at least as well as dense GTs while training up to $100\times$ faster, with speedups that grow with graph size (Appendix~\ref{app:runtime}). In Figure \ref{fig:transferability_main}, sparse GT with RPEARL PEs exhibit superior generalization across models, with the largest margin on SNAP-Patents (0.58 vs. 0.51 for the next best model). On Squirrel and OGBN-MAG, sparse GTs with RPEARL PEs are close to their $\alpha = 1$ accuracy even at the smallest training fractions, whereas the dense GT needs substantially larger training graphs to reach comparable accuracy. Figure~\ref{fig:heatmap_transferability} further supports this advantage: at a fixed training graph size, sparse GTs sustain high accuracy as the testing graph grows, including when it exceeds the training graph size. We conjecture that the $k$-hop mask supplies local structure that partially compensates for the instability of eigenvector PEs.

\begin{figure*}[t]
\centering
    \includegraphics[width = 0.9\linewidth]{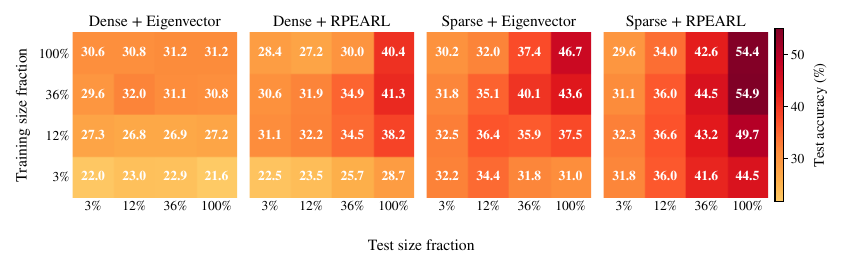}
    \caption{Transferability heatmaps on SNAP-Patents. The x-axis is the ratio of testing graph size to the full graph size while y-axis is the ratio of training graph sizes. The color represents the accuracy.}
    \label{fig:heatmap_transferability}
\end{figure*}

\subsection{Terrain graph application} \label{sec:experiments_terrain_graph}







A suitable application of our theoretical results emerges in the domain of distance estimation over terrain graphs. We follow the setup of~\citep{chenDecoupledNeuroGFShortest2025}, a metric learning problem over grid graphs to estimate shortest path distances. Each dataset comprises a high-resolution grid graph $\bbG$. We downsample $\bbG$ into coarser grid graphs $\bbG_r$ using strides $r\in\{4,5,\dots,40\}$. For each $\bbG_r$, we train Sparse GT with an MSE loss on the $L_1$ approximation of node embeddings against its ground truth SPDs. After training, we evaluate Sparse GT on the full resolution $\bbG$ on a new set of sample points. For hyperparameters and more implementation details, see Appendix~\ref{app:terrain_graphs}.

We present the results on the Norway~\citep{kartverket_dem_2025} and Philadelphia~\citep{chenDecoupledNeuroGFShortest2025} datasets in Figure~\ref{fig:tg_transferability} and provide a GNN baseline for reference. We observe that both architectures show strong transferability, achieving low errors on all training fractions. In particular, we observe that Sparse GT trained with as few as $10,000$ nodes shows performance comparable to training with an order of magnitude larger graphs. When averaged across fractions, we observe a 25.2\% NMAE improvement relative to GNNs in Norway, and a 26.4\% improvement in Philadelphia.
\begin{figure*}[h]
\centering
    \includegraphics[width = 0.9\linewidth]{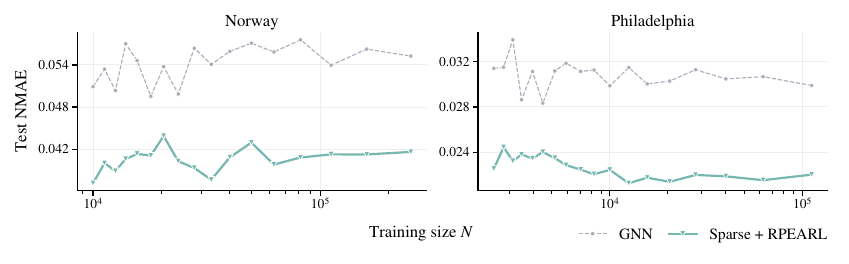}
    \caption{Transferability on terrain graphs measured by normalized mean absolute error (NMAE), lower is better $\downarrow$. The $x$-axis corresponds to the number of nodes for each training resolution, and the $y$-axis is the NMAE on the highest resolution test graph.}
    \label{fig:tg_transferability}
\end{figure*}

\section{Conclusions} \label{sec:conclusions}
We have presented a framework showing that Graph Transformers endowed with transferable positional encodings inherit their transferability. We also provided empirical verification that various GTs exhibit competitive performance on large graphs when trained on smaller graphs. These results are valuable for applications dealing with large graphs, where data collection and training costs are a major challenge. This notion of transferability inheritance can be applied to other graph limit objects, for example, random graphs sampled from graphons. Furthermore, our results motivate future research directions into other structural, relative and absolute graph positional encodings that  may be proven to be transferable if they satisfy the transferability hypothesis of Theorem~\ref{thm:gt_mt_convergence}.

\section{Limitations}
This work is scoped to the assumptions established in Section~\ref{sec:transference}. We have provided commentary discussing why we do not believe these assumptions are restrictive in practice. Additionally, we assume that all the graph datasets used in this work are sampled from an underlying, low-dimensional manifold, commonly known as the manifold hypothesis. There is empirical evidence suggesting that the spectrum of graph datasets is indeed supported on a low-dimensional subspace, for instance see~\citep{wangManifoldPerspectiveStatistical2024}, Appendices H, I.

\clearpage
\subsection*{AI use statement}
In this work, we have not used generative AI tools for help develop theoretical models or conceptual frameworks, formulate mathematical claims, provide critical ingredients for proving mathematical claims, assist in the writing of proofs, propose or refine hypotheses, design or provide feedback on research  methodology or experiments, implement methods, clean and reformat dataset, support qualitative and thematic data analysis, interpret results, and generate synthetic data sets, assist with translation are not applicable to this work. Additionally, we used generative AI tools for edit a research paper to improve readability, identify relevant literature, formulate questions for surveys or interviews, create or modify scientific figures or images, suggest experimental parameters, and create or edit software code. We have reviewed all AI-assisted work. All suggested revisions to improve the readability were reviewed and edited by the authors. LLM-assisted code was verified and tested for correctness by 2 authors. We take responsibility for the final content of this work, including text, claims or artifacts produced with the aid of generative AI.


\bibliographystyle{styles/iclr2027_conference}
\bibliography{main}

\newpage

\appendix

\onecolumn
\section{Related Work}
\label{sec_app:related_work}
\paragraph{Transferability and Generalization via Limit Models.} 
The convergence of GNNs to manifold neural networks (MNNs) has been leveraged extensively as an analytical tool to establish the transferability, stability, and statistical generalization of graph neural networks~\citep{wang2024geometric,wang2024stability,wangManifoldPerspectiveStatistical2024}. Transferability has also been explored with graphons as a limit model for graphs \citep{ruiz2020graphon, maskey2025generalization, maskey2023transferability}. \citet{keriven2023functions} studied the limit function space of GNNs over random graphs, as well as the role of positional encodings. More recently, \citet{levin2025transferring} studied a general framework for transferability by relating small graph instances to limit objects. Beyond graph models, \citet{calvello2025continuum} study the attention on general function spaces and connect discrete attention to continum operators.

\paragraph{Efficient training on large graphs.}
GraphSAGE~\citep{hamiltonInductiveRepresentationLearning2017} scales GNNs to large-scale graphs by training on minibatches of neighborhood subgraphs. \citet{linUnifiedGTUniversalFramework2024} leverages GraphSAGE-style sampling. Based on graphon convergence results, \citet{cervinoLearningTransferenceTraining2023} proposes training on growing graphs. \citet{deng2024polynormer} improves scalability through linear local-to-global attention mechanism.

\paragraph{Length Generalization in Transformers.} 
Length generalization in large language models has been studied both empirically~\citep{zhouTransformersCanAchieve2024,anilExploringLengthGeneralization2022} and theoretically~\citep{huangFormalFrameworkUnderstanding2024}. It has been argued that PEs play an important role in achieving robust length generalization~\citep{kazemnejadImpactPositionalEncoding2023,pengYaRNEfficientContext2023}. Yet, the characterization and necessary conditions for length generalization remain an open area of research. \citet{li2023transformers} presented a generalization analysis for general Transformers without considering graph structural information. Statistical generalization of graph transformers have also been studied in recent years with relative positional encodings~\citep{li2024generalizationgt}.

\paragraph{GNN-Transformer Hybrids.} Several works combine message-passing GNNs with global attention into hybrid ``local-global'' models, including GraphGPS \citep{rampavsek2022recipe}, SGFormer~\citep{wuSGFormerSimplifyingEmpowering2023}, GraphTrans~\citep{wuRepresentingLongRangeContext2022}, and UnifiedGT~\citep{linUnifiedGTUniversalFramework2024}. Our theoretical results on the transferability of GTs also cover these hybrid architectures.

\paragraph{GT Positional Encodings.} Beyond Laplacian eigenvectors, other additive encodings include learned spectral encodings drawn from the Laplacian spectrum (SAN)~\citep{kreuzer2021rethinking} or learned mappings of the random walk matrix~\citep{dwivediGraphNeuralNetworks2021}. Spectral encoders have also been used to address the eigenvector signs and eigenspace basis changes \citep{lim2023signbasis} as well as the stability under perturbations \citep{huangstability}. Learnable encoders include pretrained foundation graph encoder with positional and structural representations \citep{canturk2024gpse} as well as GNN-based positional encodings \citep{kanatsoulisLearningEfficientPositional2024}. Relative PEs encode pairwise information, such as GraphiT (diffusion kernels)~\citep{mialon2021graphit}, the Graphormer family (shortest-path and centrality biases)~\citep{ying2021transformers}, random walk approaches~\citep{maGraphInductiveBiases2023,geislerTransformersMeetDirected2023a}, or GKAT (heat-kernel-based)~\citep{choromanskiBlockToeplitzMatricesDifferential2022}. Mask-based relative encodings replace quadratic attention with structured sparsity to trade off globality and scalability, for example, Exphormer (local neighborhood + expander graphs)~\citep{shirzad2023exphormer}, and UnifiedGT ($k$-hop neighborhood)~\citep{linUnifiedGTUniversalFramework2024}. We refer to~\citep{black2024comparing} for an in-depth analysis of the properties of absolute and relative positional encodings on GTs. Recent works further study the expressive limits with truncations \citep{flora2026truncated} and combinations with rotary positional encodings \citep{reid2026wire}. \citet{stoll2025generalizable} introduce
the Generalized-Distance Transformer and analyze its representation
power in relation to attention and positional encodings.

\section{Preliminaries} \label{app:preliminaries}

\subsection{Manifold decomposition and induced manifold signal}
\label{app:induce}

The graph $\bbG$ contains points $X_N = \{x_i\}_{i=1}^N$ sampled from the manifold $\ccalM$. The graph signal $\bbX \in \reals^{q \times N}$ can be viewed as a discretization of the continuous manifold signal $f$ evaluated at points $X_N$, that is
\begin{align}
    \bbS_N f = \bbX,
\end{align}
where $\bbS_N: L^2(\mu) \rightarrow L_2(X_N)$ is called the sampling operator. The sample $X_N$ induces a decomposition of the manifold \citep{trillosErrorEstimatesSpectral2018}, $\{V_i\}_{i=1}^N$ with respect to $X_N$, with $V_i \subset B_r(x_i)$ a ball of radius $r$ centered at $x_i$, with respect to the geodesic distance $d_\ccalM$ on $\ccalM$.

Let $\mu_N=\frac{1}{N} \sum_{i=1}^N \delta_{x_i}$ be the empirical measure of the random sample. The decomposition is defined by the $\infty$-optimal transport map $T:\ccalM \rightarrow X_N$, defined by the $\infty$-optimal Transport Distance between $\mu$ and $\mu_N$,
\begin{align}
    d_\infty(\mu,\mu_N)= \min_{T:T\#\mu=\mu_N} \operatorname*{ess\,sup}_{x\in\ccalM} d_\ccalM(x, T(x)).
\end{align}

Here, $T\#\mu$ denotes the pushforward measure, i.e., $\mu(T^{-1}(V)) = \mu_N(V)$ holds for every subset $V \subseteq X_N$, and $V_i = T^{-1}(\{x_i\})$.

The radius of the balls where the partitions are contained can be bounded as $r \leq A(\frac{\log N}{N})^{1/d}$ when $d\geq 3$ and as $r \leq A(\log N)^{3/4} / N^{1/2}$ when $d=2$, with $A$ being a constant related to the geometry of the manifold.


The manifold function induced by the signals of the sampled graph is a piecewise constant function defined by 
\begin{align}
    (\bbI_N\bbX)(x) = \sum_{i=1}^N [\bbX]_i \mathbbm{1}_{x \in V_i} ,
\end{align}
where $\bbI_N: L_2(X_N) \rightarrow L^2(\mu)$ denotes the interpolation operator. To compare induced signals of graphs with different numbers of nodes, we use the norm $\|h\|_{L^{1,2}(\ccalM)} := \int_\ccalM \|h(x)\|_2\, d\mu(x)$ for $h \in L^2(\mu;\reals^q)$, i.e., the integral of the pointwise Euclidean norm.

\subsection{Spectral convergence of graph Laplacians} \label{app:spectral_convergence}
We recall the pointwise spectral convergence of the graph Laplacian to the Laplace--Beltrami operator from \citep{dunson2021spectral}.
\begin{proposition}{\citep{dunson2021spectral}[Theorem~4]}
    \label{prop:spectralconver}
    For a sufficiently small $\varepsilon_N>0$, if $N$ is sufficiently large so that $\varepsilon_N \geq \left( \frac{\log N}{N}\right)^{\frac{1}{2d+12}}$, then with probability greater than $1-N^{-2}$, for all $0 \leq i < J$, 
    \begin{equation}
        |\lambda_{i,N} -\lambda_i| \leq \Omega_1\varepsilon_N^2, 
    \end{equation}
    with $\Omega_1$ is related to the eigengap of $\ccalL$, $d$, and the diameter, the volume, the injectivity radius, the curvature and the second fundamental form of the manifold.
    \end{proposition}

\section{Proofs of Graph Transformer Transferability} \label{app:gt_proofs}

\subsection{Generalized Graph and Manifold Transformers} \label{app:generalized_gt}
In Section~\ref{sec:transference} we discuss softmax-based GTs and MTs. These come from the following generalized definitions. 

\begin{definition}[Generalized GT] \label{def:ggt}
    A GT layer is given by 
    \begin{align}
        \bbPhi_\bbG(\ccalT,\bbX)_i = \frac{\sum_{j=1}^N \gamma_{ij} \bbV \bbx_j}{\sum_{j=1}^N \gamma_{ij}}, \label{eqn:kernel_gt_attn_trf}
    \end{align}
    
where $\gamma_{ij} = \gamma(\bbx_i,\bbx_j)$, and  $\gamma: \reals^q \times \reals^q \rightarrow \reals_+$ is called the kernel function.
\end{definition}
\begin{definition}[Manifold Transformer]\label{def:mt}
    A Manifold Transformer (MT) layer is given by
    \begin{align}
        \bm\Phi_\ccalM(\ccalT,f)(x) &=
        \frac{\int_\mathcal{M} \tilde\gamma_{xy} \bbV f( y)\ d\mu( y)}
        {\int_\mathcal{M} \tilde\gamma_{xy}\ d\mu( y)} ,\label{eqn:mt_attn}
    \end{align}
    for all $x\in \ccalM$, with $\tilde\gamma_{xy}:=\gamma(f(x),f(y))$, and the input $f \in L^2(\mu;\reals^q)$ defined as the positional encoding of a manifold signal $g \in L^2(\mu;\reals^p)$, that is, $f(x) = \bm\Psi_\ccalM(\ccalH,\ccalL,g)(x)$.
\end{definition}
The Manifold Transformer is a map $\bm\Phi_\ccalM: L^2(\mu;\reals^p) \rightarrow L^2(\mu;\reals^q)$ resulting from the composition of $f$ (the positional encoding) with manifold attention. The vector $f(x)$ corresponds to the output of the positional encodings in the manifold domain. The manifold attention of Equation~\eqref{eqn:mt_attn} is the continuous analogue of Equation~\eqref{eqn:kernel_gt_attn_trf}, with infinitely many attention coefficients for each point in the manifold.

\begin{remark}
    The softmax graph attention of Equation \eqref{eq:gt_attn_trf} is the particular case of Definition~\ref{def:ggt} where attention coefficients are given by $\gamma_{ij} = \exp{\langle \bbQ \bbx_i, \bbK\bbx_j\rangle}$. 
\end{remark}

Our convergence results hold for any kernel satisfying the following assumption.

\begin{assumption}[Valid kernel function]
    \label{assm:kernel}
    The kernel function $\gamma_{xy}$ satisfies the following: (i) Nonnegative: For any $x,y \in \ccalM$, $\gamma_{xy}\geq 0$. (ii) Bounded: $\gamma_{xy} \leq C$ for some constant $C$, (iii) For $x,y, x',y' \in \ccalM$ and signal $f$, it holds that $\|\gamma_{xy} - \gamma_{x'y'} \| \leq L_\gamma \big(\|f(x)-f(x')\| + \|f(y)-f(y')\|\big)$.
\end{assumption}

The exponential kernel of softmax attention, $\gamma_{xy} = \exp\langle \bbQ f(x), \bbK f(y)\rangle$, satisfies Assumption~\ref{assm:kernel}: it is nonnegative, bounded, and Lipschitz with constant $L_\gamma = C_QC_Ke^B$, where $B := C_QC_K$, as we prove in Proposition~\ref{prop:exp_kernel}.

\paragraph{A comment on manifold attention.} Notice that, for each manifold point $x$, the normalized manifold attention weights $\gamma_x(y) := \tilde\gamma_{xy} / \int_\ccalM \tilde\gamma_{xy'}\, d\mu(y')$ form a probability density over $\ccalM$. Thus, manifold attention is an expectation over $\bbV f(y)$ with distribution $\gamma_x$ supported on the manifold, $\bm\Phi_\ccalM(\ccalT,f)(x) = \int_\ccalM \bbV f(y) \gamma_x(y)\ d\mu(y) = \mathbb{E}_{y\sim \gamma_x} \left[\bbV f(y)\right]$. Therefore, the special case of softmax attention can be interpreted as the expectation over values with distribution $\gamma_x$ a Gibbs measure with energies $E_x(y) = -\langle \bbQ f(x), \bbK f(y)\rangle$.

\subsection{Auxiliary lemmas and propositions}\label{app:lemmas_props}
In this section we state a series of lemmas and propositions that will be useful for the proof of Theorem ~\ref{thm:gt_mt_convergence}. Throughout the proofs, let $B:= C_Q C_K$, so that, by Assumptions~\ref{assm:linear_ops} and~\ref{assm:normalized_lipschitz_pe} (bounded operators and normalized encodings), \(|\langle \bbQ\bbx_i,\bbK\bbx_j\rangle|\le B,\) and $\sup_{u,v\in\ccalM}\ |\langle \bbQ f(u),\bbK f(v) \rangle| \leq B$.

\begin{proposition} [Exponential Kernel] \label{prop:exp_kernel}
    Assumption~\ref{assm:kernel} holds for the kernel 
    $\gamma(x,y) = \gamma_{xy} = \exp \left[\langle \bbQ f(x),\bbK f(y)\rangle\right]$.
    \begin{proof}
        The exponential function is nonnegative, so (i) holds. For (ii), since the signals are bounded, we can bound the kernel by $\exp \left[\langle \bbQ f(x),\bbK f(y)\rangle\right] \leq e^B$,  Finally, for (iii), let $x,y,x',y' \in \ccalM$, then we have 
    \begin{align}
       \nonumber  &\| \exp \left[\langle \bbQ f(x),\bbK f(y)\rangle\right] - \exp \left[\langle \bbQ f(x'),\bbK f(y')\rangle\right]\|  \\
        &\quad\leq e^B  \Big| \left[\langle \bbQ f(x),\bbK f(y)\rangle\right] - \left[\langle \bbQ f(x'),\bbK f(y')\rangle\right]\Big| \label{eq:prop_kernel_1} \\
        &\quad= e^B  \Big| \left[\langle \bbQ (f(x) - f(x')),\bbK f(y)\rangle\right] - \left[\langle \bbQ f(x'),
        \bbK (f(y')-f(y))\rangle\right]\Big| \label{eq:prop_kernel_2} \\
        &\quad \leq e^B C_QC_K\big[ \|f(x) - f(x') \| + \|f(y)-f(y') \|\big]  \label{eq:prop_kernel_3}.
    \end{align}
    where in \eqref{eq:prop_kernel_1} we use the mean value theorem, in \eqref{eq:prop_kernel_2} we add and subtract $\langle \bbQ f(x'), \bbK f(y)\rangle$, and in \eqref{eq:prop_kernel_3} we bound the linear operators, use $\|f\|\leq 1$ (Assumption~\ref{assm:normalized_lipschitz_pe}), and apply the triangle inequality.
    \end{proof}
\end{proposition}

\begin{lemma} \label{lem:inner_product}
    Let $x \in \ccalM$ a point in the manifold, and $y \in V_j$, a point in partition $V_j \subset B_r(x_j)$. Then it holds that
    \begin{align}
        \left| \langle \bbQ f(x),\bbK f(x_j) \rangle  -  \langle \bbQ f(x),\bbK f(y) \rangle \right|  \leq C_QC_K\ r
    \end{align}
\end{lemma}

\begin{proof}
\begin{align}
    \left| \langle \bbQ f(x),\bbK f(x_j) \rangle  -  \langle \bbQ f(x),\bbK f(y) \rangle \right| 
    &= \left| \langle \bbQ f(x),\bbK \left( f(x_j) - f(y)\right) \rangle \right| \label{prop_1_step1} \\
    &\leq C_Q \|f(x)\| \cdot C_K \|f(x_j)-f(y)\| \label{prop_1_step2} \\
    &\leq C_QC_K \|f(x_j)-f(y)\| \label{prop_1_step3} \\
    &\leq C_QC_K d_\ccalM(x_j,y) \label{prop_1_step4} \\
    &\leq C_QC_K\ r \label{prop_1_step5} 
\end{align}
Where in 
\eqref{prop_1_step2} we apply the bound on the linear operators $\bbQ$ and $\bbK$, in 
\eqref{prop_1_step3} we apply the bound $\|f(x)\|\leq 1$ from Assumption~\ref{assm:normalized_lipschitz_pe}, in 
\eqref{prop_1_step4} we apply Assumption~\ref{assm:normalized_lipschitz_pe}, and in 
\eqref{prop_1_step5} we use the fact that $y \in V_j$, therefore $d_\ccalM(y,x_j)\leq r$.
\end{proof}

For notational brevity, onwards we will denote the GT attention coefficients as $\gamma_{ij} = \exp[\langle \bbQ\bbx_i,\bbK\bbx_j \rangle]$, the MT attention coefficients as $\tilde{\gamma}_{iy} = \exp[\langle \bbQ f(x_i),\bbK f(y) \rangle]$, and the induced manifold signal coefficients as $\tilde{\gamma}_{ij} =  \exp[\langle \bbQ f(x_i),\bbK f(x_j) \rangle]$.

\begin{lemma} \label{lem:manifold_coeffs}
    Let $X_N = \{x_i\}_{i=1}^N$ be a set of points sampled from the manifold $\ccalM$, with its corresponding induced partitioning $\{V_i\}_{i=1}^N$. For each $x_j \in X_N$, and for any $y\in V_j$, it holds that
    \begin{align}
        |\tilde\gamma_{ij} - \tilde\gamma_{iy}| &\leq e^B C_QC_K\ r
    \end{align}
\end{lemma}

\begin{proof}

\begin{align}
    |\tilde\gamma_{ij} - \tilde\gamma_{iy}|
    &= |\exp{ \langle \bbQ f(x_i),\bbK f(x_j) \rangle } - \exp{ \langle \bbQ f(x_i),\bbK f(y) \rangle }| \\
    &\leq e^B \left| \langle \bbQ f(x_i),\bbK f(x_j) \rangle  - \langle \bbQ f(x_i),\bbK f(y) \rangle \right| \\
    &\leq e^B C_QC_K\ r
\end{align}

where $B = C_QC_K$ as in Appendix~\ref{app:lemmas_props}. In the first inequality we use the mean value theorem, $|e^a-e^b| \leq e^{\max\{a,b\}} |a-b|$. In the second we apply the bound from Lemma \ref{lem:inner_product}.
\end{proof}

\begin{lemma}\label{lmm:ij_to_ij}
    For each $x_i,x_j\in X_N$ it holds that
    \begin{align}
    |\gamma_{ij}-\tilde\gamma_{ij}|
    \le 2e^{B}C_{QK}\Delta_{\bbPsi(\bbG)}.
    \end{align}
\end{lemma}

\begin{proof}
    \begin{align}
    |\gamma_{ij}-\tilde\gamma_{ij}|
    &= \Bigl|\exp\!\big(\langle \bbQ\bbx_i,\bbK\bbx_j\rangle\big) - \exp\!\big(\langle \bbQ f(x_i),\bbK f(x_j)\rangle\big)\Bigr| \label{lmm:fix_1}\\
    &\le e^{B}\Bigl|\langle \bbQ\bbx_i,\bbK\bbx_j\rangle-\langle \bbQ f(x_i),\bbK f(x_j)\rangle\Bigr| \label{lmm:fix_2}\\
    &\le e^{B}\Bigl(\big|\langle \bbQ(\bbx_i-f(x_i)),\bbK\bbx_j\rangle\big|   +\big|\langle \bbQ f(x_i),\bbK(\bbx_j-f(x_j))\rangle\big|\Bigr) \label{lmm:fix_3}\\
    &\le e^{B}\Bigl(\|\bbQ\|\,\|\bbK\|\,\|\bbx_i-f(x_i)\|\,\|\bbx_j\|  +\|\bbQ\|\,\|\bbK\|\,\|f(x_i)\|\,\|\bbx_j-f(x_j)\|\Bigr) \label{lmm:fix_4}\\
    &\le 2e^{B}C_{QK}\Delta_{\bbPsi(\bbG)}. \label{lmm:1_16}
    \end{align}

    In \eqref{lmm:fix_2} we apply the mean value theorem $|e^a-e^b|\le e^{\max\{a,b\}}|a-b|$ and upper-bound $\max\{a,b\}\le B$.
    In \eqref{lmm:fix_3} we add and subtract $\langle \bbQ f(x_i),\bbK\bbx_j\rangle$ and use bilinearity and the triangle inequality.
    In \eqref{lmm:fix_4} we apply Cauchy--Schwarz and boundedness of the linear operators.
    In \eqref{lmm:1_16} we use Assumption~\ref{assm:normalized_lipschitz_pe} and the transferable-PE hypothesis of Theorem~\ref{thm:gt_mt_convergence}.
\end{proof}

\begin{lemma}\label{lem:mc_v_mc_cor}
    Let $x,y \in \ccalM$, with $d_\ccalM(x, x_i) \leq r$. Then, it holds that
    \begin{align}
        |\tilde\gamma_{iy} - \tilde\gamma_{xy} | \leq e^BC_{QK} r \label{eq:mc_v_mc_cor_bd}
    \end{align}
\end{lemma}

\begin{proof}
    \begin{align}
    |\tilde\gamma_{iy} - \tilde\gamma_{xy} | \
        &= | \exp \langle \bbQ f(x_i),\bbK f(y) \rangle - \exp \langle \bbQ f(x),\bbK f(y) \rangle | \label{eq:mc_v_mc_cor_1} \\
        &\quad \leq e^B \left|\langle \bbQ f(x_i),\bbK f(y) \rangle - \langle \bbQ f(x),\bbK f(y) \rangle \right| \label{eq:mc_v_mc_cor_2} \\
        &\quad=  e^B | \langle \bbQ \left[f(x_i) - f(x)\right],\bbK f(y) \rangle |  \label{eq:mc_v_mc_cor3}\\
        &\quad\leq e^BC_{QK} r. \label{eq:mc_v_mc_cor_4}
    \end{align}
    where in \eqref{eq:mc_v_mc_cor_2} we use $|e^a-e^b|\le e^{\max\{a,b\}}|a-b|$ and upper-bound $\max\{a,b\}\le B$, in \eqref{eq:mc_v_mc_cor3} we rearrange the terms, and in~\eqref{eq:mc_v_mc_cor_4} we use the bounds on linear operators, bounded manifold signal, normalized Lipschitz of $f$, and $d_\ccalM(x, x_i) \leq r$.
\end{proof}

\subsection{Proof of Theorem \protect\ref{thm:gt_mt_convergence}} \label{app:gt_mt_cvg_proof}

Theorem \ref{thm:gt_mt_convergence} bounds the convergence of a Graph Transformer with transferable PEs to a Manifold Transformer with the corresponding limit PEs.
In the proof we combine linear operator constants $C_{QK} := C_QC_K$ and $C_QC_KC_V = C_\ccalT$.

\begin{proof}
The Graph Transformer's output for the $i$-th node can be written in vector form as
\begin{align}
\bm\Phi_\bbG(\ccalT,\bbX)(x_i) = \frac{\sum_{j=1}^{N} \exp[\langle \bbQ\bbx_{i}, \bbK\bbx_{j} \rangle]\bbV\bbx_{j} }
               {\sum_{j=1}^{N} \exp[\langle \bbQ\bbx_{i}, \bbK\bbx_{j} \rangle]}
\end{align}
We use the attention coefficients $\gamma_{ij}$, $\tilde{\gamma}_{iy}$ and $\tilde{\gamma}_{ij}$ defined in Appendix~\ref{app:lemmas_props}. Furthermore, when necessary we will abbreviate the denominator terms as 
$D_\bbG = \sum_{j=1}^N \gamma_{ij}$, $D_\ccalM = \int_{\ccalM} \tilde\gamma_{iy}d\mu(y)$, and $\bar{D}_\ccalM = \sum_{j=1}^N \tilde{\gamma}_{ij}$.

Thus we have
\begin{align}
    \bm\Phi_\bbG(\ccalT,\bbX)(x_i) &= \frac{1}{D_\bbG}\sum_{j=1}^{N} \gamma_{ij}\bbV\bbx_{j} \\
    \bm\Phi_\ccalM(\ccalT,f)(x_i) &= \frac{1}{D_\ccalM}\int_\ccalM \tilde{\gamma}_{iy}\bbV f(y)d\mu(y)
\end{align}

We will introduce an auxiliary term that evaluates the MT kernel and values at the sampled points, denoted by
\begin{align}
    \bar{\bm\Phi}_\ccalM(\ccalT,\bbX)(x_i)
    &= \frac{\sum_{j=1}^N\int_{V_j} \tilde{\gamma}_{ij}\bbV f(x_j) d\mu(y)}{\sum_{j=1}^N\int_{V_j} \tilde{\gamma}_{ij} d\mu(y)}\label{eq:interm1}.
\end{align}
Notice that since the integrand does not depend on $y$, we can simplify $\bar{\bm\Phi}_\ccalM(\ccalT,\bbX)(x_i)$ to:
\begin{align}
    \bar{\bm\Phi}_\ccalM(\ccalT,\bbX)(x_i)
    &= \frac{\sum_{j=1}^N\int_{V_j} \tilde{\gamma}_{ij}\bbV f(x_j) d\mu(y)}{\sum_{j=1}^N\int_{V_j} \tilde{\gamma}_{ij} d\mu(y)} \label{eq:interm2} \\
    &= \frac{\sum_{j=1}^N \tilde{\gamma}_{ij}\bbV f(x_j) \cdot \mu(V_j)}{\sum_{j=1}^N   \tilde{\gamma}_{ij} \cdot \mu(V_j)} \label{eq:interm3} \\
    &= \frac{\frac{1}{N}\sum_{j=1}^N \tilde{\gamma}_{ij}\bbV f(x_j)}{\frac{1}{N}\sum_{j=1}^N   \tilde{\gamma}_{ij}} \label{eq:interm4} \\
    &= \frac{1}{\bar{D}_\ccalM}\sum_{j=1}^N \tilde{\gamma}_{ij}\bbV f(x_j), \label{eq:interm5}
\end{align}

The output difference for node $i$ can be decomposed as
\begin{align}
    &\nonumber \| \bm\Phi_\bbG(\ccalT,\bbX)(x_i) - \bm\Phi_\ccalM(\ccalT,f)(x_i) \|  \\
    &\label{eq:thm1_2} = \| \bm\Phi_\bbG(\ccalT,\bbX)(x_i) - \bar{\bm\Phi}_\ccalM(\ccalT,\bbX)(x_i)  + \bar{\bm\Phi}_\ccalM(\ccalT,\bbX)(x_i) - \bm\Phi_\ccalM(\ccalT,f)(x_i)\|   \\
    &\leq \| \bm\Phi_\bbG(\ccalT,\bbX)(x_i) - \bar{\bm\Phi}_\ccalM(\ccalT,\bbX)(x_i)\|  + \| \bar{\bm\Phi}_\ccalM(\ccalT,\bbX)(x_i) - \bm\Phi_\ccalM(\ccalT,f)(x_i)\| \label{eq:thm1_3}
\end{align}
To Equation \eqref{eq:thm1_2} we add and subtract the induced manifold signal term, and in \eqref{eq:thm1_2} to \eqref{eq:thm1_3} we use the triangle inequality.

We will now bound the first term of \eqref{eq:thm1_3} as
\begin{align}
\left| \bm\Phi_\bbG(\ccalT,\bbX)(x_i) - \bar{\bm\Phi}_\ccalM(\ccalT,\bbX)(x_i) \right| = \left| \frac{1}{D_\bbG}\sum_{j=1}^{N} \gamma_{ij}\bbV\bbx_{j} - \frac{1}{\bar{D}_\ccalM}\sum_{j=1}^{N} \tilde{\gamma}_{ij}\bbV f(x_j)   \right|.
\end{align}

Distribute the denominators, add and subtract $D_\ccalM \gamma_{ij} \bbV f(x_j)$, then apply triangle inequality: 
\begin{align}
    &\left\| \frac{1}{D_\bbG \bar{D}_\ccalM}  \sum_{j=1}^N  \bar{D}_\ccalM\gamma_{ij}\bbV\bbx_j - D_\bbG \tilde{\gamma}_{ij}\bbV f(x_j)  \right\| \nonumber\\
    &= \Bigg\| \frac{1}{D_\bbG \bar{D}_\ccalM}  \sum_{j=1}^N \bar{D}_\ccalM\gamma_{ij}\bbV\bbx_j - \bar{D}_\ccalM \gamma_{ij} \bbV f(x_j) + \bar{D}_\ccalM \gamma_{ij} \bbV f(x_j) - D_\bbG \tilde{\gamma}_{ij}\bbV f(x_j)  \ \Bigg\| \label{eq:thm1_4} \\
    & \leq 
    \left\| \frac{1}{D_\bbG \bar{D}_\ccalM} \sum_{j=1}^N \bar{D}_\ccalM\gamma_{ij}\bbV\bbx_j - \bar{D}_\ccalM \gamma_{ij} \bbV f(x_j) \right\| + \nonumber \\&\qquad \qquad \qquad \qquad \qquad \qquad \qquad \left\| \frac{1}{D_\bbG \bar{D}_\ccalM} \sum_{j=1}^N \bar{D}_\ccalM \gamma_{ij} \bbV f(x_j) - D_\bbG \tilde{\gamma}_{ij}\bbV f(x_j) \right\|  \\
    &= 
    \left\| \frac{1}{D_\bbG } \sum_{j=1}^N \gamma_{ij}\bbV(\bbx_j -f(x_j)) \right\| +  \left\| \frac{1}{D_\bbG \bar{D}_\ccalM} \sum_{j=1}^N (\bar{D}_\ccalM \gamma_{ij} - D_\bbG \tilde{\gamma}_{ij}) \bbV f(x_j)\right\|. \label{eq:thm1_4b}
\end{align}

Note that $\bbx_j - f(x_j)$ represents the output difference between a GNN and an MNN evaluated at node $x_j$. By the transferable-PE hypothesis of Theorem~\ref{thm:gt_mt_convergence}, this difference is bounded by $\Delta_{\bbPsi(\bbG)}$. Using this bound, we can control the first term of \eqref{eq:thm1_4b} as follows:
\begin{align}
    \frac{1}{D_\bbG} \left\| \sum_{j=1}^N \gamma_{ij}\bbV(\bbx_j -f(x_j)) \right\| \leq \frac{1}{D_\bbG} \sum_{j=1}^N \gamma_{ij} \|\bbV\| \cdot \| \bbx_j - f(x_j) \| &\leq \frac{C_V \Delta_{\bbPsi(\bbG)}}{D_\bbG}  \sum_{j=1}^N \gamma_{ij} \\
    &= C_V\Delta_{\bbPsi(\bbG)}, \label{eq:thm1_gnn}
\end{align}
where we first apply the triangle inequality, along with the definition of the operator norm. Finally \eqref{eq:thm1_gnn} uses the bound on the operator norm of $\bbV$ and $\|\bbx_j - f(x_j)\| \leq \Delta_{\bbPsi(\bbG)}$ from the transferable-PE hypothesis of Theorem~\ref{thm:gt_mt_convergence}, followed by the identity $\sum_{j=1}^N \gamma_{ij} = D_\bbG$. 

Now we bound the second term in \eqref{eq:thm1_4b}. We add and subtract another comparative term $\bar{D}_\ccalM\tilde{\gamma}_{ij}$,
\begin{align}
    &\left\| \frac{1}{D_\bbG \bar{D}_\ccalM} \sum_{j=1}^N (\bar{D}_\ccalM \gamma_{ij} - D_\bbG \tilde{\gamma}_{ij}) \bbV f(x_j)\right\|  \nonumber\\
    &= \Bigg\| \frac{1}{D_\bbG \bar{D}_\ccalM}  \sum_{j=1}^N \Bigl[ (\bar{D}_\ccalM\gamma_{ij} - \bar{D}_\ccalM\tilde{\gamma}_{ij}) + (\bar{D}_\ccalM\tilde{\gamma}_{ij} - D_\bbG \tilde{\gamma}_{ij})\Bigr] \bbV f(x_j) \Bigg\| \label{eq:thm1_5} \\
    &= \Bigg\| \frac{1}{D_\bbG \bar{D}_\ccalM} \sum_{j=1}^N \Big[ \bar{D}_\ccalM(\gamma_{ij} - \tilde{\gamma}_{ij})+  (\bar{D}_\ccalM - D_\bbG)\tilde{\gamma}_{ij} \Big]\bbV f(x_j) \Bigg\| \label{eq:thm1_6} \\
    &\leq \sum_{j=1}^N \left\| \frac{\bar{D}_\ccalM}{D_\bbG \bar{D}_\ccalM} \left[ (\gamma_{ij} - \tilde{\gamma}_{ij}) \right] \bbV f(x_j) \right\| + \left\| \frac{\tilde{\gamma}_{ij}(\bar{D}_\ccalM-D_\bbG)}{D_\bbG \bar{D}_\ccalM} \bbV f(x_j) \right\| \label{eq:thm1_7} \\
    &\leq \sum_{j=1}^N \|\bbV\|\|f(x_j)\| \biggl[  \frac{1}{D_\bbG} |\gamma_{ij}  - \tilde{\gamma}_{ij} | + \frac{\tilde{\gamma}_{ij}}{D_\bbG \bar{D}_\ccalM}  | \bar{D}_\ccalM - D_\bbG |\biggr]  \label{eq:thm1_8} \\
    &\leq C_V\sum_{j=1}^N \left[  \frac{1}{D_\bbG} |\gamma_{ij}  - \tilde{\gamma}_{ij} | + \frac{\tilde{\gamma}_{ij}}{D_\bbG \bar{D}_\ccalM}  | \bar{D}_\ccalM - D_\bbG |\right]  \label{eq:thm1_8b} \\
    &\leq C_V \left[ \frac{N \cdot 2e^BC_{QK}\Delta_{\bbPsi(\bbG)}}{D_\bbG} + \frac{ | \bar{D}_\ccalM - D_\bbG | }{D_\bbG \bar{D}_\ccalM}\sum_{j=1}^N \tilde{\gamma}_{ij} \right] \label{eq:thm1_8c} \\
    &\leq C_V \left[ 2 e^{2B} C_{QK}\Delta_{\bbPsi(\bbG)}  + \frac{| \bar{D}_\ccalM - D_\bbG |}{D_\bbG} \right] \label{eq:thm1_9}
\end{align} 
On \eqref{eq:thm1_6} we rearrange the terms, on \eqref{eq:thm1_7} we apply the triangle inequality, on \eqref{eq:thm1_8} we use the operator norm inequality, and on \eqref{eq:thm1_8b} we apply Assumptions~\ref{assm:linear_ops} and \ref{assm:normalized_lipschitz_pe}. On \eqref{eq:thm1_8c} we apply the bound from Lemma \ref{lmm:ij_to_ij}. Finally, \eqref{eq:thm1_9} uses $\sum_{j=1}^N \tilde{\gamma}_{ij} = \bar{D}_\ccalM$ for the second term, and for the first term we note that the worst case denominator is $D_\bbG \geq N \cdot e^{-B}$, which we can rearrange to bound $N/D_\bbG \leq e^B$.

We can bound the remaining term using Lemma \ref{lmm:ij_to_ij}:
\begin{align}
    \frac{|\bar{D}_\ccalM - D_\bbG|}{D_\bbG} 
    &= \frac{1}{D_\bbG}\left| \sum_{j=1}^N \tilde{\gamma}_{ij} - \sum_{j=1}^N \gamma_{ij} \right| \\
    &\leq \frac{1}{D_\bbG}\sum_{j=1}^N |\tilde{\gamma}_{ij} - \gamma_{ij}| \label{thm:1_15} \\
    &\leq \frac{N}{D_\bbG} \cdot 2e^B C_{QK} \Delta_{\bbPsi(\bbG)} \label{thm:1_16} \\
    &\leq 2e^{2B} C_{QK} \Delta_{\bbPsi(\bbG)}, \label{thm:1_17}
\end{align}
where in \eqref{thm:1_15} we apply the triangle inequality, in \eqref{thm:1_16} we apply Lemma \ref{lmm:ij_to_ij}, and in \eqref{thm:1_17} we again use $N/D_\bbG \leq e^B$. 

Substituting \eqref{thm:1_17} into \eqref{eq:thm1_9}, the second term of \eqref{eq:thm1_4b} is bounded by
\begin{align}
\left\| \frac{1}{D_\bbG \bar{D}_\ccalM} \sum_{j=1}^N (\bar{D}_\ccalM \gamma_{ij} - D_\bbG \tilde{\gamma}_{ij}) \bbV f(x_j)\right\| 
    &\leq C_V\left[ 2e^{2B} C_{QK} \Delta_{\bbPsi(\bbG)} + 2e^{2B} C_{QK} \Delta_{\bbPsi(\bbG)} \right] \\
    & = 4 e^{2B} C_\ccalT \Delta_{\bbPsi(\bbG)},\label{eq:thm1_second_term_bound}
\end{align}
where we absorb the linear operator constants into $C_\ccalT = C_{QK}C_V$. Combining \eqref{eq:thm1_gnn} and \eqref{eq:thm1_second_term_bound},  we conclude the bound for the first term of \eqref{eq:thm1_3} is 
\begin{align}
\left\| \bm\Phi_\bbG(\ccalT,\bbX)(x_i) - \bar{\bm\Phi}_\ccalM(\ccalT,\bbX)(x_i) \right\|
    &\leq C_V\Delta_{\bbPsi(\bbG)} +  4 e^{2B} C_\ccalT \Delta_{\bbPsi(\bbG)}. \label{eq:thm1_full_bound_1st_term}
\end{align}
We now bound the second term of \eqref{eq:thm1_3}.  Let $\tilde{D}_\ccalM = \sum_{j=1}^{N} \int_{V_j}\tilde{\gamma}_{ij}d\mu(y)$ denote the discrete approximation of $D_\ccalM$ in integral form,
\begin{align}
    \nonumber &\| \bar{\bm\Phi}_\ccalM(\ccalT,\bbX)(x_i) - \bm\Phi_\ccalM(\ccalT,f)(x_i) \| \\
    & = \Bigg\| \frac{1}{\tilde{D}_\ccalM}\sum_{j=1}^{N} \int_{V_j}\tilde{\gamma}_{ij}\bbV f(x_j)d\mu(y) - \frac{1}{D_\ccalM} \int_\ccalM \tilde{\gamma}_{iy}\bbV f(y)d\mu(y) \Bigg\| \label{thm_1_st_2} \\
    & = \Bigg\| \frac{1}{\tilde{D}_\ccalM D_\ccalM} \bigg( D_\ccalM\sum_{j=1}^{N} \int_{V_j}\tilde{\gamma}_{ij}\bbV f(x_j)d\mu(y)  - \tilde{D}_\ccalM \int_\ccalM \tilde{\gamma}_{iy}\bbV f(y)d\mu(y) \bigg) \Bigg\| \label{thm_1_st_3} \\
    & = \Bigg\| \frac{1}{\tilde{D}_\ccalM D_\ccalM} \bigg( D_\ccalM\sum_{j=1}^{N} \int_{V_j}\tilde{\gamma}_{ij}\bbV f(x_j)d\mu(y)  - \tilde{D}_\ccalM \sum_{j=1}^{N}\int_{V_j} \tilde{\gamma}_{iy}\bbV f(y)d\mu(y) \bigg) \Bigg\| \label{thm_1_st_4} \\
    & = \Bigg\| \frac{1}{\tilde{D}_\ccalM D_\ccalM} \sum_{j=1}^{N} \int_{V_j} \bigg( D_\ccalM\tilde{\gamma}_{ij}\bbV f(x_j) - \tilde{D}_\ccalM \tilde{\gamma}_{iy}\bbV f(y) \bigg) d\mu(y) \Bigg\| \label{thm_1_st_5} \\
    &\leq \frac{1}{\tilde{D}_\ccalM D_\ccalM} \sum_{j=1}^{N} \int_{V_j} \Big\|D_\ccalM\tilde{\gamma}_{ij}\bbV f(x_j)- \tilde{D}_\ccalM \tilde{\gamma}_{iy}\bbV f(y) \Big\| d\mu(y) \label{eq:thm_1_st_6}
\end{align}

In \eqref{thm_1_st_3} we distribute the denominators, in \eqref{thm_1_st_4} we partition the integral, in \eqref{thm_1_st_5} we rearrange terms, and in \eqref{eq:thm_1_st_6} we apply the triangle inequality.

To bound the integrand of \eqref{eq:thm_1_st_6}, we add and subtract $D_\ccalM\tilde{\gamma}_{ij}\bbV f(y)$ and apply the triangle inequality:
\begin{align}
   \nonumber  &\left\| D_\ccalM\tilde{\gamma}_{ij}\bbV f(x_j) - \tilde{D}_\ccalM \tilde{\gamma}_{iy}\bbV f(y) \right\| \\
    &\leq \left\| D_\ccalM \tilde{\gamma}_{ij}\bbV f(x_j) - D_\ccalM\tilde{\gamma}_{ij}\bbV f(y) \right\|  + \left\| D_\ccalM\tilde{\gamma}_{ij}\bbV f(y) -  \tilde{D}_\ccalM \tilde{\gamma}_{iy}\bbV f(y) \right\|  \label{eq:thm_1_st_7} \\
    &\leq D_\ccalM \left\| \bbV (f(x_j) - f(y)) \right\| + \left\| ( D_\ccalM \tilde{\gamma}_{ij} - \tilde{D}_\ccalM \tilde{\gamma}_{iy} ) \bbV f(y) \right\|. \label{eq:thm_1_st_8}
\end{align}

The first term of \eqref{eq:thm_1_st_8} is bounded by
\begin{align}
    D_\ccalM \left\| \bbV (f(x_j) - f(y)) \right\|  
    &\leq D_\ccalM \|\bbV\| \|f(x_j) - f(y)\| \label{eq:thm_1_st_10} \\
    &\leq D_\ccalM C_V \|f(x_j) - f(y)\| \label{eq:thm_1_st_11} \\
    &\leq D_\ccalM C_V r. \label{eq:thm_1_st_12}
\end{align}
In \eqref{eq:thm_1_st_10} we use linear operator bound, in
\eqref{eq:thm_1_st_11} we use Assumption \ref{assm:linear_ops}, and in
\eqref{eq:thm_1_st_12} we use the fact that $\|f(x_j) - f(y)\| \leq r$ for all $y \in V_j$ and Assumption~\ref{assm:normalized_lipschitz_pe} ($f$ is normalized Lipschitz continuous).

The second term of \eqref{eq:thm_1_st_8} is bounded by
\begin{align}
   \left\| ( D_\ccalM \tilde{\gamma}_{ij} - \tilde{D}_\ccalM \tilde{\gamma}_{iy} ) \bbV f(y) \right\| 
    &\leq \left| D_\ccalM \tilde{\gamma}_{ij} - \tilde{D}_\ccalM \tilde{\gamma}_{iy} \right| \| \bbV \| \| f(y) \| \label{eq:thm_1_st_13} \\
    &\leq C_V\left| D_\ccalM \tilde{\gamma}_{ij} - \tilde{D}_\ccalM \tilde{\gamma}_{iy} \right| \label{eq:thm_1_st_14} \\
    &= C_V\left| D_\ccalM \tilde{\gamma}_{ij} - D_\ccalM\tilde{\gamma}_{iy} + D_\ccalM\tilde{\gamma}_{iy} - \tilde{D}_\ccalM \tilde{\gamma}_{iy} \right| \label{eq:thm_1_st_15} \\
    &\leq C_V\left( D_\ccalM \left| \tilde{\gamma}_{ij} - \tilde{\gamma}_{iy} \right| + \left| D_\ccalM - \tilde{D}_\ccalM \right| \tilde{\gamma}_{iy} \right) \label{eq:thm_1_st_16}
\end{align}

In \eqref{eq:thm_1_st_13} we apply the linear operator bound, in \eqref{eq:thm_1_st_14} we use Assumption \ref{assm:linear_ops} ($\|\bbV\| \leq C_V$) and $\|f(y)\| \leq 1$ (Assumption~\ref{assm:normalized_lipschitz_pe}), in \eqref{eq:thm_1_st_15} we add and subtract $D_\ccalM\tilde{\gamma}_{iy}$, and in \eqref{eq:thm_1_st_16} we apply the triangle inequality and factor.

To bound the difference between denominator terms:
\begin{align}
    \left| D_\ccalM - \tilde{D}_\ccalM \right| 
    &= \left| \int_\ccalM \tilde{\gamma}_{iy}d\mu(y) - \sum_{j=1}^{N} \int_{V_j} \tilde{\gamma}_{ij}d\mu(y) \right| \label{eq:thm_1_st_18}\\
    &= \left| \sum_{j=1}^{N} \int_{V_j} (\tilde{\gamma}_{iy} - \tilde{\gamma}_{ij})d\mu(y) \right| \label{eq:thm_1_st_19}\\
    &\leq \sum_{j=1}^{N} \int_{V_j} \left| \tilde{\gamma}_{iy} - \tilde{\gamma}_{ij} \right| d\mu(y) \label{eq:thm_1_st_20}\\
    &\leq \sum_{j=1}^{N} \int_{V_j} e^B  C_{QK} r \ d\mu(y) \label{eq:thm_1_st_21}\\
    &= e^B  C_{QK} r , \label{eq:thm_1_st_22}
\end{align}
where in \eqref{eq:thm_1_st_19} we partition the integral, in \eqref{eq:thm_1_st_20} we apply the triangle inequality, and in \eqref{eq:thm_1_st_21} we apply Lemma \ref{lem:manifold_coeffs} and the fact that $\mu(\ccalM) = 1$.

Now, applying \eqref{eq:thm_1_st_22} and Lemma \eqref{lem:manifold_coeffs} to \eqref{eq:thm_1_st_16} we have that the second term of \eqref{eq:thm_1_st_8} is bounded by 
\begin{align}
\left\| ( D_\ccalM \tilde{\gamma}_{ij} - \tilde{D}_\ccalM \tilde{\gamma}_{iy} ) \bbV f(y) \right\| 
    &\leq C_V\big( D_\ccalM e^B C_QC_K\ r + e^BC_{QK}r \tilde{\gamma}_{iy} \big) \\
    &\leq D_\ccalM e^BC_\ccalT r +e^BC_\ccalT r \tilde{\gamma}_{iy} \label{eq:thm_1_st_23}
\end{align}
Now, we return to bound the integral of \eqref{eq:thm_1_st_6},

\begin{align}
    &\| \bar{\bm\Phi}_\ccalM(\ccalT,\bbX)(x_i) - \bm\Phi_\ccalM(\ccalT,f)(x_i) \|  \\
    &\leq\frac{1}{\tilde{D}_\ccalM D_\ccalM} \sum_{j=1}^{N} \int_{V_j} \Big\|D_\ccalM\tilde{\gamma}_{ij}\bbV f(x_j) - \tilde{D}_\ccalM \tilde{\gamma}_{iy}\bbV f(y) \Big\| d\mu(y) \label{eq:thm_1_st_24}\\
    &\leq \frac{1}{\tilde{D}_\ccalM D_\ccalM} \sum_{j=1}^{N} \int_{V_j} \Big[D_\ccalM C_V r                                 + D_\ccalM e^BC_\ccalT r                         + e^BC_\ccalT r \tilde{\gamma}_{iy}\Big] d\mu(y) \label{eq:thm_1_st_25}\\
    &\leq \frac{1}{\tilde{D}_\ccalM D_\ccalM} \Big[D_\ccalM C_V r \mu(\ccalM)              + D_\ccalM e^BC_\ccalT r \mu(\ccalM)               + e^BC_\ccalT r \sum_{j=1}^{N} \int_{V_j} \tilde{\gamma}_{iy}d\mu(y)\Big] \label{eq:thm_1_st_26}\\
   &= \frac{C_Vr\mu(\ccalM)}{\tilde{D}_\ccalM}
                                 + \frac{e^BC_\ccalT r\mu(\ccalM)}{\tilde{D}_\ccalM}
                                 + \frac{e^BC_\ccalT r\mu(\ccalM)}{D_\ccalM} \\
   &\leq \frac{C_Vr\mu(\ccalM)}{e^{-B}\mu(\ccalM)} 
                                 + \frac{e^BC_\ccalT r\mu(\ccalM)}{e^{-B}\mu(\ccalM)}  
                                 + \frac{e^BC_\ccalT r\mu(\ccalM)}{e^{-B}\mu(\ccalM)} \label{eq:thm_1_st_27} \\
   &\leq e^BC_Vr + 2e^{2B}C_\ccalT r \label{eq:thm_1_st_28}
\end{align}

where in \eqref{eq:thm_1_st_25} we use the bounds of \eqref{eq:thm_1_st_12} and \eqref{eq:thm_1_st_23}, 
in \eqref{eq:thm_1_st_26} we evaluate the integrals of the first two terms,
in \eqref{eq:thm_1_st_27} we distribute the denominators and use the fact that both denominators are positive quantities greater than $e^{-B}\mu(\ccalM)$. 

We conclude that the second term of \eqref{eq:thm1_3} is bounded by
\begin{align}
\| \bar{\bm\Phi}_\ccalM(\ccalT,\bbX)(x_i) - \bm\Phi_\ccalM(\ccalT,f)(x_i)\| \leq e^BC_Vr + 2e^{2B}C_\ccalT r \label{eq:thm_1_st_bound}
\end{align}

Finally, we can bound \eqref{eq:thm1_3} using \eqref{eq:thm1_full_bound_1st_term} and \eqref{eq:thm_1_st_bound}, by gathering terms. With this we conclude that the output difference of GT and MT is bounded by
\begin{align}
\| \bm\Phi_\bbG(\ccalT,\bbX)(x_i) - \bm\Phi_\ccalM(\ccalT,f)(x_i) \| 
    &\leq C_V\Delta_{\bbPsi(\bbG)} +  4 e^{2B}  C_\ccalT \Delta_{\bbPsi(\bbG)} + e^BC_Vr + 2e^{2B}C_\ccalT r \\
    &\leq \big[C_V + 4e^{2B}C_\ccalT\big]\Delta_{\bbPsi(\bbG)}+\ \big[ e^BC_V + 2 e^{2B} C_\ccalT \big] r. \label{eq:thm_1_final_bd_proof}
\end{align}
which gives us the statement of Theorem \ref{thm:gt_mt_convergence}.
\end{proof}
\subsection{Proof of Corollary \protect\ref{cor:gt_transferability}} \label{app:gt_transferability_proof}

Corollary \ref{cor:gt_transferability} bounds the output difference between two GTs trained with differently sized graphs by applying Theorem \ref{thm:gt_mt_convergence}.

\begin{proof}
The output difference can be bounded as

\begin{align}
    &\frac{1}{\mu(\ccalM)}  \| \bbI_{N_1}\bm\Phi_{\bbG_1}(\ccalT,\bbX_1) - \bbI_{N_2}\bm\Phi_{\bbG_2}(\ccalT,\bbX_2) \|_{L^{1,2}(\ccalM)} \nonumber\\
    & =\frac{1}{\mu(\ccalM)} \| \bbI_{N_1}\bm\Phi_{\bbG_1}(\ccalT,\bbX_1) - \bm\Phi_\ccalM(\ccalT,f)  + \bm\Phi_\ccalM(\ccalT,f) - \bbI_{N_2}\bm\Phi_{\bbG_2}(\ccalT,\bbX_2) \|_{L^{1,2}(\ccalM)} \label{eq:cor_proof_1} \\
    &\leq \frac{1}{\mu(\ccalM)}\| \bbI_{N_1}\bm\Phi_{\bbG_1}(\ccalT,\bbX_1) - \bm\Phi_\ccalM(\ccalT,f) \|_{L^{1,2}(\ccalM)} \label{eq:cor_proof_2} + \frac{1}{\mu(\ccalM)}\| \bm\Phi_\ccalM(\ccalT,f) - \bbI_{N_2}\bm\Phi_{\bbG_2}(\ccalT,\bbX_2) \|_{L^{1,2}(\ccalM)}  \\
    & \leq \frac{1}{\mu(\ccalM)}\int_\ccalM \|  \bbI_{N_1}\bm\Phi_{\bbG_1}(\ccalT,\bbX_1)(x) - \bm\Phi_\ccalM(\ccalT,f)(x) \|_2\ d\mu(x) \label{eq:cor_proof_3} \\
    &\quad+ \frac{1}{\mu(\ccalM)}\int_\ccalM \| \bm\Phi_\ccalM(\ccalT,f)(x) -  \bbI_{N_2}\bm\Phi_{\bbG_2}(\ccalT,\bbX_2)(x) \|_2\ d\mu(x) \label{eq:cor_main_decomp_n1_n2}
\end{align}

where in \eqref{eq:cor_proof_1} we add and subtract $\bm\Phi_\ccalM(\ccalT,f)$, in \eqref{eq:cor_proof_2} we apply the triangle inequality, and in \eqref{eq:cor_proof_3} we use the definition of $L^{1,2}(\ccalM)$.

The two terms in \eqref{eq:cor_main_decomp_n1_n2} correspond to the pointwise difference between the induced manifold signal of GT and the output signal of the MT, for $x\in\ccalM$. Consider bounding the first term, that compares $\bbPhi_{\bbG_1}$ with $\bbPhi_\ccalM$,

\begin{align}
    &\frac{1}{\mu(\ccalM)}\int_\ccalM \|  \bbI_{N_1}\bm\Phi_{\bbG_1}(\ccalT,\bbX_1)(x) - \bm\Phi_\ccalM(\ccalT,f)(x) \|_2\ d\mu(x) \nonumber \\
    &=\frac{1}{\mu(\ccalM)} \sum_{i=1}^{N_1} \int_{V_i} \|  \bm\Phi_{\bbG_1}(\ccalT,\bbX_1)(x_i) - \bm\Phi_\ccalM(\ccalT,f)(x) \|_2\ d\mu(x)   \label{eq:cor_proof_4}\\
    &=\frac{1}{\mu(\ccalM)} \sum_{i=1}^{N_1} \int_{V_i} \|  \bm\Phi_{\bbG_1}(\ccalT,\bbX_1)(x_i) - \bm\Phi_\ccalM(\ccalT,f)(x_i)  +\bm\Phi_\ccalM(\ccalT,f)(x_i)  - \bm\Phi_\ccalM(\ccalT,f)(x) \|_2\ d\mu(x)  \\
    &\leq\frac{1}{\mu(\ccalM)} \sum_{i=1}^{N_1} \int_{V_i} \|  \bm\Phi_{\bbG_1}(\ccalT,\bbX_1)(x_i) - \bm\Phi_\ccalM(\ccalT,f)(x_i) \|_2 + \| \bm\Phi_\ccalM(\ccalT,f)(x_i)  - \bm\Phi_\ccalM(\ccalT,f)(x) \|_2\ d\mu(x)  \label{eq:cor_proof_5}
\end{align}

In~\eqref{eq:cor_proof_4} we apply the interpolation operator on $\bbPhi_{\bbG_1}$, where $\{V_i\}_{i=1}^{N_1}$ is the partition induced by $\bbG_1$, then in~\eqref{eq:cor_proof_5} we add and subtract the MT output at point $x_i$, and apply the triangle inequality.  The first term of \eqref{eq:cor_proof_5} is the statement of Theorem \ref{thm:gt_mt_convergence}. The second term is the output difference of MT between a point within a partition $x\in V_i$ and the center of the ball containing the partition $x_i$, therefore it holds that $\|x_i -x \| \leq r$. Denote the softmax denominators over $x$ and $x_i$ as $D_\ccalM(x)$ and $D_\ccalM(x_i)$ respectively. We can decompose this as
\begin{align}
    &\| \bm\Phi_\ccalM(\ccalT,f)(x_i)  - \bm\Phi_\ccalM(\ccalT,f)(x) \|_2   \nonumber\\
    &\leq \Bigg\| \frac{\int_\ccalM \tilde{\gamma}_{iy}\bbV f(y)d\mu(y)}{\int_\ccalM \tilde{\gamma}_{iy}d\mu(y)} - \frac{\int_\ccalM \tilde{\gamma}_{xy}\bbV f(y)d\mu(y)}{\int_\ccalM \tilde{\gamma}_{xy}d\mu(y)} \Bigg\| \\
    &\leq \Bigg\| \frac{1}{D_\ccalM(x)D_\ccalM(x_i)} \int_\ccalM \Big( D_\ccalM(x)\tilde{\gamma}_{iy}\bbV f(y) -  D_\ccalM(x_i)\tilde{\gamma}_{xy}\bbV f(y) \Big) d\mu(y) \Bigg\| \\
    &\leq \frac{1}{D_\ccalM(x)D_\ccalM(x_i)} \int_\ccalM \Big\| D_\ccalM(x)\tilde{\gamma}_{iy}\bbV f(y) -  D_\ccalM(x_i)\tilde{\gamma}_{xy}\bbV f(y) \Big\|\ d\mu(y) \label{eq:cor_proof_5b}
\end{align}

Which we obtain by distributing the denominators, applying the triangle inequality, and grouping terms. The integrand in \eqref{eq:cor_proof_5b} is bounded as 
\begin{align}
    &\left\| D_\ccalM(x)\tilde{\gamma}_{iy}\bbV f(y) -  D_\ccalM(x_i)\tilde{\gamma}_{xy}\bbV f(y) \right\| \nonumber\\
    &\leq \| D_\ccalM(x) \tilde\gamma_{iy} \left[ \bbV f(x) - \bbV f(y) \right] - \left[ D_\ccalM(x)\tilde\gamma_{iy} - D_\ccalM(x_i) \tilde\gamma_{xy} \right]\bbV f(y) \| \label{eq:cor_proof_5c}  \\
    &\leq D_\ccalM(x)C_V\tilde\gamma_{iy}\|f(x) - f(y)\| + C_V \big| D_\ccalM(x)\tilde\gamma_{iy} - D_\ccalM(x_i) \tilde\gamma_{xy} \big| \label{eq:cor_proof_6}
\end{align}

where in~\eqref{eq:cor_proof_5c} we add and subtract $D_\ccalM(x)\tilde{\gamma}_{iy}\bbV f(x)$, and in~\eqref{eq:cor_proof_6} we apply the triangle inequality and the bound on $\bbV$.

We now focus on the second term of \eqref{eq:cor_proof_6},
\begin{align}
\| D_\ccalM(x)\tilde\gamma_{iy} - D_\ccalM(x_i) \tilde\gamma_{xy} \| \leq  D_\ccalM(x)|\tilde\gamma_{iy} - \tilde\gamma_{xy} | + \| \left[D_\ccalM(x) - D_\ccalM(x_i)\right] \tilde\gamma_{xy} \| \label{eq:cor_proof_7}
\end{align}

The second term of \eqref{eq:cor_proof_7} is
\begin{align}
 \tilde\gamma_{xy} \| D_\ccalM(x) - D_\ccalM(x_i) \| 
    &= \left\| \int_\ccalM \tilde\gamma_{xy} \tilde\gamma_{xy} - \tilde\gamma_{iy} d\mu(y) \right\| \\
    &\leq \int_\ccalM \tilde\gamma_{xy} \left\| \tilde\gamma_{xy} - \tilde\gamma_{iy} \right\| d\mu(y) \\
    &\leq D_\ccalM(x_i)e^BC_{QK}r \label{eq:cor_proof_9}
\end{align}

where we use Lemma \ref{lem:mc_v_mc_cor} to bound the remaining integral by $D_\ccalM(x_i)$.

Returning to the integral \eqref{eq:cor_proof_5b}, we have
\begin{align}
    &\frac{C_V}{D_\ccalM(x)D_\ccalM(x_i)} \int_\ccalM \Big\| D_\ccalM(x)\tilde{\gamma}_{iy}\bbV f(y) -  D_\ccalM(x_i)\tilde{\gamma}_{xy}\bbV f(y) \Big\|\ d\mu(y) \label{eq:cor_proof_integral_1} \\
    &\qquad\leq \frac{C_V}{D_\ccalM(x)D_\ccalM(x_i)} \int_\ccalM \Big\| D_\ccalM(x)e^BC_{QK}r + \tilde\gamma_{xy} \mu(\ccalM)e^BC_{QK}r \Big\|\ d\mu(y) \\
    &\qquad= \frac{C_V}{D_\ccalM(x)D_\ccalM(x_i)}  D_\ccalM(x)\mu(\ccalM)e^BC_{QK}r + \tilde\gamma_{xy} D_\ccalM(x)\mu(\ccalM)e^BC_{QK}r \\
    &\qquad= \frac{1}{D_\ccalM(x_i)}  2\mu(\ccalM)e^BC_\ccalT r \\
    &\qquad\leq 2e^{2B}C_\ccalT r
\end{align}

by applying the bounds of \eqref{eq:cor_proof_9} and Lemma~\ref{lem:mc_v_mc_cor}, then evaluating the integral, and finally using the fact that the denominator is $D_\ccalM(x_i)\geq e^{-B}\mu(\ccalM)$. We conclude that 
\begin{align}
    \| \bm\Phi_\ccalM(\ccalT,f)(x_i)  - \bm\Phi_\ccalM(\ccalT,f)(x) \|_2 \leq 2e^{2B}C_\ccalT r \label{eq:cor_2nd_term_bd}
\end{align}

We can now bound \eqref{eq:cor_proof_3} by applying this bound and the bound of Theorem \ref{thm:gt_mt_convergence}, 
\begin{align}
 \frac{1}{\mu(\ccalM)}\int_\ccalM \|  \bbI_{N_1}\bm\Phi_{\bbG_1}(\ccalT,\bbX_1)(x) - \bm\Phi_\ccalM(\ccalT,f)(x) \|_2\ d\mu(x)  \leq \Delta_{\bbPhi(\bbG_1)} + 2e^{2B}C_\ccalT\, r_1,
\end{align}
where $\Delta_{\bbPhi(\bbG_1)}$ denotes the bound of Theorem \ref{thm:gt_mt_convergence} for $\bbG_1$, and $r_1$ its partition radius.

Applying this bound in \eqref{eq:cor_main_decomp_n1_n2} for $\bm\Phi(\bbX_1)$ and $\bm\Phi(\bbX_2)$ gives us the statement of Corollary \ref{cor:gt_transferability}.


\end{proof}

\section{Convergence of Positional Encodings} \label{app:pe_convergence}
\subsection{Proof of Theorem \protect\ref{thm:lap_pe_convergence}}
\label{app:lap_pe_cvg_proof}


\begin{proof}
    Our setting is the same as Theorem 4 in \citet{dunson2021spectral}, we restate their conclusion in our notations as:
    with probability greater than $1-N^{-2}$ there are
signs $r_{j,N}\in\{-1,1\}$, simultaneously for $1\leq j\leq K$, such that
\begin{equation}
    \max_{1\leq i\leq N}
 \left| r_{j,N} [\bbphi_{j,N}]_i -\bm\phi_j(x_i)
 \right|
 \leq  \frac{\Omega_{K}}{\Gamma_K}\varepsilon_N^2,
\end{equation}
with $\Gamma_K:= \min_{1\leq j\leq K} (\lambda_{j+1} - \lambda_j)$ and $\varepsilon_N \geq \left(\frac{\log N}{N}\right)^{1/(2d+8)}$, 
and $\Omega_{K}$ depends on $\lambda_K$, $d$, and the diameter, the volume, the injectivity radius, the curvature and the second fundamental form of the manifold.

Define
$\bbR_N=\operatorname{diag}(r_{1,N},\ldots,r_{K,N})$, for every $x_i\in X_N$,
\begin{align}
    \bigl\|\bbR_N\bbPsi_\bbG^{(K)}(\bbL_N)(x_i)-\bbPsi_\ccalM^{(K)}(\ccalL)(x_i)\bigr\|^2_2
 &=\sum_{j=1}^{K}
   |r_{j,N} [\bbphi_{j,N}]_i -\bm\phi_j(x_i)|^2\\
 &\leq K\,\Omega^2_{K}\,\Gamma_K^{-2}\,\varepsilon_N^4.
\end{align}
Taking square roots yields the bound.
\end{proof}

\subsection{Graph and Manifold Eigengap}\label{app:eigengaps}
In Section~\ref{sec:lap_pe_convergence} we discussed the dependence of the eigenvector PE convergence on the Laplace-Beltrami operator's eigengap. Here we discuss the relationship between $\Gamma_{K,N}$ and $\Gamma_{K}$. Using Proposition~\ref{prop:spectralconver}, we have 
\begin{align}
    \left|\left(\lambda_{j+1,N}-\lambda_{j,N}\right)-\left(\lambda_{j+1}-\lambda_j\right)\right|\leq 2\,\Omega_1\,\varepsilon_N^2,\qquad 1\leq j\leq K,
\end{align}

Thus the gaps are related by 
\begin{align}
    \left|\Gamma_{K,N}-\Gamma_K\right|\leq 2\,\Omega_1\,\varepsilon_N^2
\end{align}

Hence $\Gamma_{K,N}$ is a reliable estimate of $\Gamma_K$ only when the sampling error
$2\Omega_1\varepsilon_N^2$ is small relative to $\Gamma_K$, which is precisely the regime
$N\geq N_0(K,\delta)$ where Theorem~\ref{thm:lap_pe_convergence} holds.

\subsection{Manifold neural networks and their relation to GNNs}\label{app:gnn_to_mnn}

\textbf{Laplace-Beltrami operator and manifold convolutions.} The manifold with probability density function $\rho$ is equipped with a weighted Laplace operator \citep{grigor2006heat}, generalizing the Laplace-Beltrami operator as
\begin{equation}
    \label{eqn:weight-Laplace}
    \ccalL_\rho f = -\frac{1}{2\rho} \text{div}(\rho^2 \nabla f),
\end{equation}
with $\text{div}$ denoting the divergence operator of $\ccalM$ and $\nabla$ denoting the gradient operator of $\ccalM$ \citep{bronstein2017geometric, gross2023manifolds}. The manifold convolution operation is defined relying on the Laplace operator $\ccalL$ \citep{wang2024stability}. For notational simplicity, we write $\ccalL_\rho = \ccalL$.
For a function $g\in L^2(\mu)$ as input,
a manifold convolutional filter \citep{wang2024stability} can be defined as
\begin{align} \label{eqn:manifold-convolution}
   f(x) = \bbh(\ccalL) g(x) =\sum_{k=0}^{K-1} h_ke^{-k\ccalL}g(x) \text{,}
\end{align}
with $h_k\in\reals$ the filter parameters. The polynomial counterpart of the graph filter in Equation~\eqref{eqn:gnn_matrix_form_pe} is
\begin{align}\label{eqn:manifold-convolution2}
    f(x) = \bbh(\ccalL) g(x) = \sum_{k=0}^{K-1} h_k \ccalL^k g(x).
\end{align}

\paragraph{Manifold neural networks.} A manifold neural network (MNN) is constructed by cascading $L$ layers, each of which contains a bank of manifold convolutional filters and a pointwise nonlinearity $\sigma
:\reals \rightarrow\reals$. The output manifold function of each layer $l=1,2\cdots, L$ can be explicitly denoted as
\begin{equation}\label{eqn:mnn}
f_l^p(x) = \sigma\left( \sum_{q=1}^{F_{l-1}} \bbh_l^{pq}(\ccalL) f_{l-1}^q(x)\right),
\end{equation}
where $f_{l-1}^q$, $1 \leq q \leq F_{l-1}$ is the $q$-th input feature from layer $l-1$ and $f_l^p$, $1 \leq p \leq F_l$ is the $p$-th output feature of layer $l$. The input to the first layer is the manifold signal, $f_0^q = g^q$, and the MNN output is the last layer, $f = f_L$. 
We denote MNN  as a mapping $f = \bm\Psi_\ccalM(\ccalH,\ccalL,g)$
, where $\ccalH =\{\bbh_{l}^{pq}\}_{l,p,q}$ is a collective set of filter parameters in all the manifold convolutional filters.

\subsection{Proof of Theorem \protect\ref{thm:gnn_mnn_convergence}}
\label{app:gnn_mnn_cvg_proof}

We impose the following additional assumption for the GNN to MNN convergence proof.

\begin{assumption}[Normalized Lipschitz signals] \label{assm:manifold_signal}
The input manifold signals $g$ are normalized Lipschitz for all points $a,b \in \ccalM$, $\|g(b)-g(a)\| \leq \|b-a\|$.
\end{assumption}

\begin{assumption}[Normalized Lipschitz nonlinearity] \label{assm:nonlinearity}
The nonlinearity $\sigma$ is normalized Lipschitz continuous, i.e.,
\[
|\sigma(a) - \sigma(b)| \leq |a - b|, \qquad \sigma(0) = 0.
\]
\end{assumption}

\begin{assumption}[Non-amplifying filters] \label{assm:nonamplifying}
The filter frequency response $\hat{h}: \mathbb{R}^+ \to \mathbb{R}$ is non-amplifying. I.e., for all $\lambda \in (0,\infty)$, $\hat{h}$ satisfies $|\hat{h}(\lambda)| \leq 1$.
\end{assumption}

\begin{assumption}[Spectral continuity of the filter] \label{assm:filter}
The frequency response function of the filter satisfies
\[
|\hat{h}(\lambda)| = \mathcal{O}(\lambda^{-d}), 
\qquad 
|\hat{h}'(\lambda)| \leq C_L \lambda^{-d-1}, 
\quad \lambda \in (0,\infty),
\]
with $C_L$ a spectral continuity constant that regularizes the smoothness of the filter function.
\end{assumption}

Assumptions \ref{assm:manifold_signal}--\ref{assm:filter} are mild assumptions on the properties of the underlying manifold, manifold signals, and filters, and are common in the analysis of Riemannian manifolds for GNN transferability. Building on this foundation, we present a convergence theorem from the literature, tailored to our setting.


Theorem~\ref{thm:gnn_mnn_convergence} provides the bound $\Delta_{\bbPsi(\bbG)}$ on the output difference between a GNN and its MNN limit required by the transferable-PE hypothesis of Theorem~\ref{thm:gt_mt_convergence}; its proof follows.

\begin{proof}[Proof of Theorem~\ref{thm:gnn_mnn_convergence}]
 The spectral representation of manifold filters are similar to graph convolutional filter, while we consider the case in which the Laplace operator is self-adjoint, positive-semidefinite and the manifold $\ccalM$ is compact. In this case, $\ccalL$ has real, nonnegative and discrete eigenvalues $\{\lambda_i\}_{i=1}^\infty$, written as $\ccalL \bm\phi_i =\lambda_i \bm\phi_i$ where $\bm\phi_i$ is the eigenfunction associated with eigenvalue $\lambda_i$. The eigenvalues are ordered in increasing order as $0=\lambda_1\leq \lambda_2\leq \lambda_3\leq \hdots$, and the eigenfunctions are orthonormal and form an eigenbasis of $L^2(\mu)$. When mapping a manifold signal onto the eigenbasis $ [\hat{g} ]_i=\langle g, \bm\phi_i\rangle_{ \ccalM} = \int_\ccalM g(x) \bm\phi_i(x)\text{d}\mu(x)$, the manifold convolution can be seen in the spectral domain as
\begin{align}
    [\hat{f}]_i =\hat{h}(\lambda_i)  [\hat{g}]_i\text{.}
\end{align}
Hence, the frequency response of manifold filter is given by $\hat{h}(\lambda)$. We will use the spectral convergence of the graph Laplacian to the Laplace--Beltrami operator stated in Proposition~\ref{prop:spectralconver} (Appendix~\ref{app:spectral_convergence}).

Because $\{x_1, x_2,\cdots,x_N\}$ is a set of randomly sampled points from $\ccalM$, based on Theorem 19 in \citep{von2008consistency} we can claim that
\begin{equation}
   \left|\langle \bbS_N g,\bbS_N \bm\phi_{i } \rangle   -\langle g,\bm\phi_i\rangle_\ccalM\right| = O\left(\sqrt{\frac{\log (1/\delta)}{N}}\right),
\end{equation}
where $\langle g,\bm\phi_i \rangle = \int_\ccalM g(x) \bm\phi_i(x) \text{d}\mu(x)$ is defined as the inner product over manifold $\ccalM$.
This also indicates that 
\begin{equation}
   \left|\|\bbS_N g \|^2-\|g\|^2_\ccalM\right| = O\left(\sqrt{\frac{\log (1/\delta)}{N}}\right),
\end{equation}
which indicates $\|\bbS_N g\|=\|g\|_\ccalM + O((\log (1/\delta)/N)^{1/4})$, where $\|g\|^2_\ccalM = \langle g, g\rangle_\ccalM$.
We suppose that the input manifold signal is $\lambda_J$-bandlimited with $J$ spectral components. We first write out the difference on each node $x_j\in X_N$ as
\begin{align}
    \|[\bbh(\bbL_N)\bbS_N g]_j - (\bbh(\ccalL) g)(x_j) \| 
    &= \Bigg\| \sum_{i=1}^N \hat{h}(\lambda_{i,N}) \langle \bbS_N g,\bm\phi_{i,N} \rangle [\bm\phi_{i,N}]_j  - \sum_{i=1}^J \hat{h}(\lambda_i)\langle g,\bm\phi_i\rangle_{\ccalM}  \bm\phi_i(x_j)  \Bigg\| \\ 
    &\leq  \Bigg\| \sum_{i=1}^J \hat{h}(\lambda_{i,N}) \langle \bbS_N g,\bm\phi_{i,N} \rangle [\bm\phi_{i,N}]_j - \sum_{i=1}^J \hat{h}(\lambda_i) \langle g,\bm\phi_{i}  \rangle_{\ccalM} \bm\phi_{i}(x_j) \nonumber \\
   &\qquad\qquad\qquad\qquad\qquad\qquad+ \sum_{i=J+1}^N \hat{h}(\lambda_{i,N}) \langle \bbS_N g,\bm\phi_{i,N} \rangle [\bm\phi_{i,N}]_j \Bigg\| \\
   & \leq \Bigg\| \sum_{i=1}^J \hat{h}(\lambda_{i,N}) \langle \bbS_N g,\bm\phi_{i,N} \rangle [\bm\phi_{i,N}]_j - \sum_{i=1}^J \hat{h}(\lambda_i) \langle g,\bm\phi_{i}  \rangle_{\ccalM}\bm\phi_{i}(x_j)\Bigg\| \nonumber \\  
   &\qquad\qquad\qquad\qquad\qquad + \left\|\sum_{i=J+1}^N \hat{h}(\lambda_{i,N}) \langle \bbS_N g,\bm\phi_{i,N} \rangle [\bm\phi_{i,N}]_j\right\|.\label{eqn:1}
\end{align}

 The first part of \eqref{eqn:1} can be decomposed with the triangle inequality as  follows
 \begin{align}
     \nonumber &\bigg\| \sum_{i=1}^J \hat{h}(\lambda_{i,N}) \langle \bbS_N g,\bm\phi_{i,N}  \rangle [\bm\phi_{i,N}]_j - \sum_{i=1}^J \hat{h}(\lambda_i)\langle g,\bm\phi_i\rangle_{\ccalM} \bm\phi_i(x_j)  \bigg\|
     \\ 
     & \leq  \left\| \sum_{i=1}^J \left(\hat{h}(\lambda_{i,N})- \hat{h}(\lambda_i) \right) \langle \bbS_N g,\bm\phi_{i,N} \rangle [\bm\phi_{i,N}]_j \right\|  +\left\| \sum_{i=1}^J \hat{h}(\lambda_i)\left( \langle \bbS_N g,\bm\phi_{i,N} \rangle  [\bm\phi_{i,N}]_j - \langle g,\bm\phi_i \rangle_{\ccalM} \bm\phi_i(x_j) \right)  \right\| .\label{eqn:conv-1}
 \end{align}

In \eqref{eqn:conv-1}, the first part relies on the difference of eigenvalues and the second part depends on the eigenvector difference. 
The first term in \eqref{eqn:conv-1} is bounded with Cauchy-Schwartz inequality as
\begin{align}
\Bigg\| \sum_{i=1}^J (\hat{h}(\lambda_{i,N} ) - \hat{h}(\lambda_i)) \langle \bbS_N g,\bm\phi_{i,N} \rangle [\bm\phi_{i,N}]_j  \Bigg\|    
     &\leq \sum_{i=1}^J \left|\hat{h}(\lambda_{i,N} )-\hat{h}(\lambda_i)\right| |\langle \bbS_N g,\bm\phi_{i,N}  \rangle | \label{eqn:p1-1}\\
     &\leq \|\bbS_N g\| \sum_{i=1}^J |\hat{h}'(\lambda_i)| |\lambda_{i,N} -\lambda_i| \nonumber \\
     &\leq \|\bbS_N g\|  \sum_{i=1}^J C_L \Omega_1 \epsilon^2  \lambda_i^{-d} \label{eqn:p2-1}\\
     &\leq  \|\bbS_N g\|C_L \Omega_1 \epsilon^2  \sum_{i=1}^J  i^{-2} \label{eqn:p2}\\
     &\leq \left( \|g\|_\ccalM+ \left(\frac{\log (1/\delta)}{N}\right)^{\frac{1}{4}}\right) C_L \Omega_1 \epsilon^2 \frac{\pi^2}{6} := A_1(N)
\end{align}
In \eqref{eqn:p1-1}, it depends on the inequality that $|[\bm\phi_{i,N}]_j| \leq \|\bm\phi_{i,N}\|_\infty \leq \|\bm\phi_{i,N}\|_2 = 1$. In \eqref{eqn:p2-1}, it depends on the filter continuity assumption.
In \eqref{eqn:p2}, we implement Weyl's law \citep{arendt2009weyl} which indicates that eigenvalues of Laplace operator scales with the order  $\lambda_i \sim i^{2/d}$. The last inequality comes from the fact that $\sum_{i=1}^\infty i^{-2}=\frac{\pi^2}{6}$.
The second term in \eqref{eqn:conv-1} can be bounded with the triangle inequality  as
{ 
\begin{align}
  & \nonumber \Bigg\| \sum_{i=1}^J \hat{h}(\lambda_i)\left( \langle \bbS_N g,\bm\phi_{i,N}  \rangle [\bm\phi_{i,N}]_j- \langle g,\bm\phi_i \rangle_{\ccalM} \bm\phi_i(x_j)\right)  \Bigg\|\\
   & \label{eqn:term1}\leq \Bigg\|  \sum_{i=1}^J \hat{h}(\lambda_i)  \left(\langle \bbS_N g,\bm\phi_{i,N} \rangle [\bm\phi_{i,N}]_j  - \langle \bbS_N g,\bm\phi_{i,N}  \rangle  \bm\phi_i(x_j)\right)\Bigg\| \nonumber \\
   &\qquad \qquad \qquad + \left\| \sum_{i=1}^J  \hat{h}(\lambda_i) \left(\langle \bbS_N g,\bm\phi_{i,N} \rangle  \bm\phi_i(x_j) -\langle g,\bm\phi_i\rangle_\ccalM \bm\phi_i(x_j) \right) \right\|
\end{align}}
The first term in \eqref{eqn:term1} can be bounded with inserting the eigenfunction convergence result in Proposition \ref{prop:spectralconver} as
\begin{align}
& \nonumber \left\|  \sum_{i=1}^J \hat{h}(\lambda_i) \left(\langle \bbS_N g,\bm\phi_{i,N} \rangle [\bm\phi_{i,N}]_j  - \langle \bbS_N g,\bm\phi_{i,N} \rangle \bm\phi_i(x_j)\right)\right\|\\
& \leq \sum_{i=1}^{J} \left|\hat{h}(\lambda_i)\right|\|\bbS_N g\| \left|[\bm\phi_{i,N}]_j - \bm\phi_i(x_j)\right|\\
&\leq \sum_{i=1}^J (\lambda_i^{-d+1}) \Omega_2 \epsilon^2\left(\|g\|_\ccalM + \left(\frac{\log (1/\delta)}{N}\right)^{\frac{1}{4}}\right)\\
&\leq \Omega_2 \epsilon^2\sum_{i=1}^J (\lambda_i^{-d+1}) \left(\|g\|_\ccalM + \left(\frac{\log (1/\delta)}{N}\right)^{\frac{1}{4}}\right)\\
&:= A_2(J,N).
\end{align}
Considering the filter continuity assumption, the second term in \eqref{eqn:term1} can be written as
\begin{align}
     & \nonumber \Bigg\| \sum_{i=1}^J \hat{h}(\lambda_{i,N} ) (\langle \bbS_N g,\bm\phi_{i,N}\rangle  \bm\phi_i (x_j) -\langle g,\bm\phi_i\rangle_\ccalM \bm\phi_i(x_j) ) \Bigg\| \\
   &\leq \sum_{i=1}^J \left|\hat{h}(\lambda_{i,N}) \right| \left|\langle \bbS_N g,\bm\phi_{i,N} \rangle   -\langle g,\bm\phi_i\rangle_\ccalM\right||\bm\phi_i(x_j)|\\
   &\leq \sum_{i=1}^J (\lambda_{i,N}^{-d})\left|\langle \bbS_N g,\bm\phi_{i,N} \rangle   -\langle g,\bm\phi_i\rangle_\ccalM\right| \|\bm\phi_i\|\\
   &\label{eqn:A3}\leq 
   \sum_{i=1}^J \left(1 + \Omega_1 \epsilon^2\right)^{-d}  \lambda_i^{-d}\left|\langle \bbS_N g,\bm\phi_{i,N} \rangle   -\langle g,\bm\phi_i\rangle_\ccalM\right|  \\
   &\leq \frac{\pi^2}{6}\left|\langle \bbS_N g,\bm\phi_{i,N} \rangle   -\langle g,\bm\phi_i\rangle_\ccalM\right| :=A_3(N)
\end{align}

The term $\left|\langle \bbS_N g,\bm\phi_{i,N} \rangle   -\langle g,\bm\phi_i\rangle_\ccalM\right| $ can be decomposed by inserting a term $\langle \bbS_N g, \bbS_N \bm\phi_i\rangle $ as
\begin{align}
    \big|\langle \bbS_N g,\bm\phi_{i,N} \rangle  -\langle g,\bm\phi_i\rangle_\ccalM\big|  &\leq  \big|\langle \bbS_N g,\bm\phi_{i,N} \rangle  - \langle \bbS_N g, \bbS_N \bm\phi_i\rangle +  \langle \bbS_N g, \bbS_N \bm\phi_i\rangle  -\langle g,\bm\phi_i\rangle_\ccalM \big| \\
    &\leq   \big|\langle \bbS_N g,\bm\phi_{i,N} \rangle  - \langle \bbS_N g, \bbS_N \bm\phi_i\rangle\big|  + \big|\langle \bbS_N g, \bbS_N \bm\phi_i\rangle  -\langle g,\bm\phi_i\rangle_\ccalM\big|\\
    & \leq \|\bbS_N g\|\| \bm\phi_{i,N} - \bbS_N\bm\phi_i\|  + \big|\langle \bbS_N g, \bbS_N \bm\phi_i\rangle  -\langle g,\bm\phi_i\rangle_\ccalM\big|\\
    &\leq \bigg(\|g\|_\ccalM + \Big(\frac{\log (1/\delta)}{N}\Big)^{\frac{1}{4}}\bigg) \frac{C_{\ccalM,2}\lambda_i {\epsilon}}{\theta_i}+ \sqrt{\frac{\log (1/\delta)}{N}}
\end{align}
Then equation \eqref{eqn:A3} can be bounded as 
\begin{align}
     & \nonumber \Bigg\| \sum_{i=1}^J \hat{h}(\lambda_{i,N} ) (\langle \bbS_N g,\bm\phi_{i,N}\rangle \bm\phi_i(x_j) -\langle g,\bm\phi_i\rangle_\ccalM \bm\phi_i(x_j) ) \Bigg\| \\
     &\leq 
   \sum_{i=1}^J (1 +\Omega_1\epsilon^2)^{-d}  (\lambda_i^{-d}) \left(\left(\|g\|_\ccalM + \left(\frac{\log (1/\delta)}{N}\right)^{\frac{1}{4}}\right) \frac{C_{\ccalM,2}\lambda_i {\epsilon}}{\theta_i} + \sqrt{\frac{\log (1/\delta)}{N}}\right)
   \\&\leq \frac{\pi^2}{6} \max_{i=1,\cdots,J}\frac{C_{\ccalM,2} \epsilon}{\theta_i}\left(\|g\|_\ccalM + \left(\frac{\log (1/\delta)}{N}\right)^{\frac{1}{4}}\right) + \frac{\pi^2}{6} \sqrt{\frac{\log (1/\delta)}{N}}
\end{align}

The second term in \eqref{eqn:1} can be bounded with the eigenvalue difference bound in Proposition \ref{prop:spectralconver} as
\begin{align}
     \left\|\sum_{i=J+1}^N \hat{h}(\lambda_{i,N}) \langle \bbS_N g,\bm\phi_{i,N} \rangle [\bm\phi_{i,N}]_j\right\|  
      &\leq \sum_{i=J+1}^N (\lambda_{i,N}^{-d})\left(\|g\|_\ccalM+\left(\frac{\log (1/\delta)}{N}\right)^{\frac{1}{4}}\right)\\
    &\leq \sum_{i=J+1}^\infty (\lambda_{i,N}^{-d})  \|g\|_\ccalM
    \\&\leq  \left(1 +\Omega_1\epsilon^2\right)^{-d} \sum_{i=J+1}^\infty (\lambda_i^{-d})  \|g\|_\ccalM\\
    &\leq  J^{-1}\|g\|_\ccalM:= A_4(J).
\end{align}

We note that the bound is made up by terms $A_1( N)+A_2(J,N)+A_3( N)+A_4(J)$, related to the bandwidth of manifold signal $J$ and the number of sampled points $N$. 
This makes the bound scale with the order
{ 
\begin{align}
\label{eqn:filter-bound}
 \|[\bbh(\bbL_N)\bbS_N g]_j -  \bbh(\ccalL) g (x_j)\|  \leq C_1' \epsilon^2 +  C_2' \epsilon \theta_J^{-1}  + C_3' \sqrt{\frac{\log(1/\delta)}{N}} + C_4' J^{-1},
\end{align}}
with $C_1' = C_L\Omega_1\frac{\pi^2}{6}\|g\|_\ccalM$, $C_2' = \Omega_2\frac{\pi^2}{6}$, $C_3' =\frac{\pi^2}{6}$ and $C_4' = \|g\|_\ccalM$.
As $N$ goes to infinity, for every $\delta >0$, there exists some $J_0$, such that for all $J>J_0$ it holds that $A_4(J)\leq \delta/2$. There also exists $n_0$, such that for all $N>n_0$, it holds that $A_1(N)+A_2(J_0, N)+A_3(N)\leq \delta/2$. We can conclude that the summations converge as $N$ goes to infinity. We see $J$ large enough to have $J^{-1}\leq \delta'$, which makes the eigengap $\theta_J$ also bounded by $\epsilon$. We combine the first two terms as 
\begin{align}
\|[\bbh(\bbL_N)\bbS_N g]_j -  \bbh(\ccalL) g(x_j)\|  \leq (C_1C_L+C_2) {\epsilon}^2  +\frac{\pi^2}{6} \sqrt{\frac{\log(1/\delta)}{N}},
\end{align}
with $C_1 = \Omega_1 \frac{\pi^2}{6}\|g\|_\ccalM$ and $C_2 =  \Omega_2 \frac{\pi^2}{6} \theta^{-1}_{\delta'^{-1}}$.
To bound the output difference of MNNs, we need to write in the form of features of the final layer
 \begin{align}
     \|[\bm\Psi_\bbG(\ccalH,\bbL_N,\bbS_N g)]_j-\bm\Psi_\ccalM(\ccalH,\ccalL, g)(x_j)\| &= \left\| \sum_{q=1}^{F } [\bbx_{L}^q]_j-\sum_{q=1}^{F } f_L^q (x_j)\right\|
     \\ &\leq \sum_{q=1}^{F } \left\| [\bbx_{L}^q]_j - f_L^q(x_j) \right\|.
 \end{align}
By inserting the definitions, we have 
 \begin{align}
    \left\| [\bbx_{l}^p]_j - f_l^p(x_j) \right\| =\left\| \sigma\left(\left[\sum_{q=1}^{F } \bbh_l^{pq}(\bbL_N) \bbx_{l-1}^q\right]_j \right) -\sigma\left(\sum_{q=1}^{F } \bbh_l^{pq}(\ccalL) f_{l-1}^q(x_j)\right) \right\|
 \end{align}
 with $\bbx_{0}=\bbS_N g$ as the input of the first layer. With a normalized point-wise Lipschitz nonlinearity, we have
  \begin{align}
    \| [\bbx_{l}^p]_j - f_l^p(x_j) & \| \leq \left\|  \sum_{q=1}^{F } \left[\bbh_l^{pq}(\bbL_N) \bbx_{l-1}^q \right]_j   -  \sum_{q=1}^{F } \bbh_l^{pq}(\ccalL)  f_{l-1}^q(x_j)\right\|\\
    & \leq \sum_{q=1}^{F } \left\|    [\bbh_l^{pq}(\bbL_N) \bbx_{l-1}^q]_j    -   \bbh_l^{pq}(\ccalL)  f_{l-1}^q(x_j)\right\|
 \end{align}
 The difference can be further decomposed as
\begin{align}
   \nonumber  & \|    [\bbh_l^{pq}(\bbL_N)  \bbx_{l-1}^q]_j    -   \bbh_l^{pq}(\ccalL)  f_{l-1}^q(x_j) \| 
   \\  &\leq \|
[\bbh_l^{pq}(\bbL_N) \bbx_{l-1}^q]_j  - [\bbh_l^{pq}(\bbL_N) \bbS_N f_{l-1}^q]_j +[\bbh_l^{pq}(\bbL_N) \bbS_N f_{l-1}^q]_j -    \bbh_l^{pq}(\ccalL)  f_{l-1}^q(x_j)
    \|\\ 
   & \leq \left\|
    [\bbh_l^{pq}(\bbL_N) \bbx_{l-1}^q]_j  - [\bbh_l^{pq}(\bbL_N)  \bbS_N f_{l-1}^q]_j
    \right\|
   +
    \left\|
    [\bbh_l^{pq}(\bbL_N)  \bbS_N f_{l-1}^q]_j  -  \bbh_l^{pq}(\ccalL)  f_{l-1}^q(x_j)
    \right\|
\end{align}
The second term can be bounded with \eqref{eqn:filter-bound} and we denote the bound as $\Delta_N$ for simplicity. The first term can be decomposed by Cauchy-Schwartz inequality and non-amplifying of the filter functions as
 \begin{align}
 \left\| [\bbx_{l}^p]_j - f_l^p(x_j) \right\| \leq \sum_{q=1}^{F } \Delta_{N}   \| \bbx_{l-1}^q\| + \sum_{q=1}^{F } \| [\bbx_{l-1}^q]_j - f_{l-1}^{q}(x_j) \|.
 \end{align}
To solve this recursion, we need to compute the bound for $\|\bbx_l^p\|$. By normalized Lipschitz continuity of $\sigma$ and the fact that $\sigma(0)=0$, we can get
 \begin{align}
  \| \bbx_l^p \|\leq \left\| \sum_{q=1}^{F } \bbh_l^{pq}(\bbL_N) \bbx_{l-1}^{q}  \right\| \leq  \sum_{q=1}^{F }  \left\| \bbh_l^{pq}(\bbL_N)\right\|  \|\bbx_{l-1}^{q}  \|   \leq   \sum_{q=1}^{F }   \|\bbx_{l-1}^{q}  \| \leq F^{l-1} \| \bbx  \|.
 \end{align}
 Insert this conclusion back to solve the recursion, we can get
 \begin{align}
 \left\| [\bbx_{l}^p]_j - f_l^p(x_j) \right\| \leq l   F^{l-1} \Delta_{N}  \|\bbx \|.
 \end{align}
 Replace $l$ with $L$ we can obtain
 \begin{align}
 \|[\bm\Psi_\bbG(\ccalH,\bbL_N,\bbS_N g)]_j- \bm\Psi_\ccalM(\ccalH,\ccalL, g)(x_j)\|  \leq   L F^{L-1}\Delta_{N}     ,
 \end{align}
when the input graph signal is normalized.
By replacing $g= \bbI_N \bbx$, we can conclude the proof.
\end{proof}

\subsection{Local Lipschitz continuity of MNNs}
\label{app:lipschitz-mnn}
We utilize Proposition 3 in \citep{wangManifoldPerspectiveStatistical2024}, which shows that the outputs of MNN defined in \eqref{eqn:mnn} are locally Lipschitz continuous within a certain area, which is stated explicitly as follows.
\begin{proposition}(Local Lipschitz continuity of MNNs \citep{wangManifoldPerspectiveStatistical2024}[Proposition~3])\label{prop:mnn-continuity}
Assume that Assumptions~\ref{assm:manifold_signal}--\ref{assm:filter} hold. Let MNN be $L$ layers with $F$ features in each layer, suppose the manifold filters are nonamplifying with $|\hat{h}(\lambda)|\leq 1$ and the nonlinearities normalized Lipschitz continuous, then there exists a constant $C'$ such that 
\begin{equation}
    \label{eqn:continuity-mnn}
    \|\bbPsi_\ccalM(\ccalH,\ccalL,g)(x) - \bbPsi_\ccalM(\ccalH,\ccalL,g)(y)\|\leq F^L C' \,d_\ccalM(x,y),\quad \text{for all }x,y \in B_r(\ccalM),
\end{equation}
where $B_r(\ccalM)$ is a ball with radius $r$ over $\ccalM$ with respect to the geodesic distance.
\end{proposition}

\subsection{RPEARL Positional Encodings}\label{app:rpearl}
In the experiments of Section~\ref{sec:experiments} we train RPEARL positional encodings, which consist of a GNN processing random i.i.d. features. In this section we define the architecture explicitly. A single pass of a GNN is given by: 
\begin{align}
    \bbP = \sigma \bigg[\sum_{k=0}^{K-1} \bbH_k\bbZ\bbL^k\bigg] \label{eqn:gnn_matrix_form}
\end{align}
where $\bbZ \in \reals^{q\times N}$ are the input signals, and $\ccalH = \{\bbH_0,\dots,\bbH_{K-1}\}$ collects the filter coefficients. Equation~\eqref{eqn:gnn_matrix_form} performs a graph convolution of order $K$ on the graph signals by the graph Laplacian $\bbL$, followed by the pointwise nonlinearity $\sigma$.

Consider taking as input single-row matrices of random node features, $\bbZ =[\bbz^{(m)}]^\top$, where $\bbz^{(m)}$ is sampled from an $N$-dimensional isotropic Gaussian. Random features play the role of node IDs, which are known to improve the expressiveness of GNNs by breaking their structural symmetries~\citep{youIdentityawareGraphNeural2021}. However, the resulting output for a single pass loses permutation equivariance. We can recover this property by drawing $M$ realizations $\{\bbz^{(m)}\}_{m=1}^M$ and taking the empirical average over independent passes of the GNN,
\begin{align}
        \bbPsi_\bbG(\ccalH,\bbL,\bbZ)  &=  \frac{1}{M}\sum_{m=1}^M\bbP^{(m)}, \qquad \bbP^{(m)} = \sigma \bigg[\sum_{k=0}^{K-1} \bbH_k\bbz^{(m)\top}\bbL^k\bigg],\label{eq:rpearl_pe}
\end{align}
where in this case the filters are $\bbH_k \in \reals^{q\times 1}$. The resulting PE $\bbPsi_\bbG(\ccalH,\bbL,\bbZ)$, named RPEARL, was shown by~\citep{kanatsoulisLearningEfficientPositional2024} to provably enhance the expressive power of GNNs beyond the WL-test 
while preserving permutation equivariance and the stability and scalability of GNNs~\citep{ruiz2021graph}. We consider RPEARL positional encodings in our experiments as an empirical enhancement for expressivity.

\section{Transferability Corollaries} \label{app:composed_corollaries}
\subsection{GT transferability corollaries}

\begin{corollary}[Transferability of GT with eigenvector PE]\label{cor:lap_transferability}
    Let $\bbG_1,\bbG_2$ satisfy the assumptions of Theorems~\ref{thm:gt_mt_convergence} and~\ref{thm:lap_pe_convergence}, and let the positional encodings be the sign-aligned eigenvectors $\bbR_{N_1}\bbPsi^{(K)}_{\bbG_1}$ and $\bbR_{N_2}\bbPsi^{(K)}_{\bbG_2}$, rescaled by a constant $c_K\leq 1$ depending on $K$ so that Assumption~\ref{assm:normalized_lipschitz_pe} holds (the rescaling does not increase the PE error of Theorem~\ref{thm:lap_pe_convergence}). Then, with probability at least $1-\delta-N_1^{-2}-N_2^{-2}$,
    \begin{align}
        &\frac{1}{\mu(\ccalM)}\left\|\bbI_{N_1}\bm\Phi_{\bbG_1}(\ccalT,\bbX_1)
        -\bbI_{N_2}\bm\Phi_{\bbG_2}(\ccalT,\bbX_2)\right\|_{L^{1,2}(\ccalM)}\nonumber\\
        &\qquad\leq C_{\bbPhi,1}\frac{\sqrt{K}\,\Omega_K}{\Gamma_K}\,(\varepsilon_{N_1}^2+\varepsilon_{N_2}^2)\nonumber\\
        &\qquad\quad+C_{\bbPhi,2}(r_1+r_2),
    \end{align}
    where $C_{\bbPhi,1}:=C_V+4e^{2B}C_\ccalT$ and $C_{\bbPhi,2}:=e^B C_V+6e^{2B}C_\ccalT$.
\end{corollary}

\begin{corollary}[Transferability of GT with GNN PEs]\label{cor:gnn_transferability}
    Let $\bbG_1,\bbG_2$ satisfy the assumptions of Theorems~\ref{thm:gt_mt_convergence} and~\ref{thm:gnn_mnn_convergence}, with the same GNN parameters $\ccalH$ and input manifold signal $g$. Then, with probability at least $1-\delta$,
    \begin{align}
        &\frac{1}{\mu(\ccalM)}\left\|\bbI_{N_1}\bm\Phi_{\bbG_1}(\ccalT,\bbX_1)
        -\bbI_{N_2}\bm\Phi_{\bbG_2}(\ccalT,\bbX_2)\right\|_{L^{1,2}(\ccalM)}\nonumber\\
        &\qquad\leq C_{\bbPhi,1}\,
        \left[C\left(\varepsilon_{N_1}^2+\varepsilon_{N_2}^2\right)
        +\frac{\pi^2}{6}\left(\sqrt{\frac{\log(2/\delta)}{N_1}}+\sqrt{\frac{\log(2/\delta)}{N_2}}\right)\right]\nonumber\\
        &\qquad\quad+C_{\bbPhi,2}(r_1+r_2),
    \end{align}
    with $C_{\bbPhi,1}, C_{\bbPhi,2}$ defined as in Corollary~\ref{cor:lap_transferability}.
\end{corollary}

Contrasting the bounds in Corollaries \ref{cor:lap_transferability} and \ref{cor:gnn_transferability} makes evident the tradeoffs of each positional encoding in the transferability of the GT. Both bounds depend on the Lipschitz constants $C_{\bbPhi,1}$ and $C_{\bbPhi,2}$ which we can control via regularization. Corollary~\ref{cor:lap_transferability} also depends on the manifold eigengap $\Gamma_K$, which, as we saw in Section~\ref{sec:lap_pe_convergence}, we do not have access to. In contrast, in Corollary~\ref{cor:gnn_transferability} we obtain the GNN constant $C$, which can also be controlled via regularization.

\subsection{Proof of Corollary \protect\ref{cor:lap_transferability}} \label{app:lap_transferability_proof}

\begin{proof}

By Theorem~\ref{thm:lap_pe_convergence}, the PEs are bounded by 
\begin{align}
    \|[\bbX_k]_i-f(x_i)\|_2
    &\leq \Delta_{\bbPsi(\bbG_k)}
    = \frac{\sqrt{K}\,\Omega_K}{\Gamma_K}\,\varepsilon_{N_k}^2.
    \label{eq:cor_lap_proof_pe_bound}
\end{align}

The proof is a straightforward application into Corollary~\ref{cor:gt_transferability}, 
\begin{align}
    &\frac{1}{\mu(\ccalM)}\left\|\bbI_{N_1}\bm\Phi_{\bbG_1}(\ccalT,\bbX_1)
    -\bbI_{N_2}\bm\Phi_{\bbG_2}(\ccalT,\bbX_2)\right\|_{L^{1,2}(\ccalM)} \nonumber\\
    &\quad\leq \Delta_{\bbPhi(\bbG_1)}+\Delta_{\bbPhi(\bbG_2)}
    +4e^{2B}C_\ccalT(r_1+r_2) \label{eq:cor_lap_proof_cor2}\\
    &\quad= \big[C_V+4e^{2B}C_\ccalT\big]
    \big(\Delta_{\bbPsi(\bbG_1)}+\Delta_{\bbPsi(\bbG_2)}\big) \nonumber\\
    &\qquad\quad+\big[e^BC_V+2e^{2B}C_\ccalT+4e^{2B}C_\ccalT\big](r_1+r_2)
    \label{eq:cor_lap_proof_collect}\\
    &\quad= C_{\bbPhi,1}\big(\Delta_{\bbPsi(\bbG_1)}+\Delta_{\bbPsi(\bbG_2)}\big)
    +C_{\bbPhi,2}(r_1+r_2), \label{eq:cor_pe_transferability_bound}
\end{align}
where in \eqref{eq:cor_lap_proof_cor2} we apply Corollary~\ref{cor:gt_transferability}, in \eqref{eq:cor_lap_proof_collect} we substitute $\Delta_{\bbPhi(\bbG_1)}$ and $\Delta_{\bbPhi(\bbG_2)}$ from Theorem~\ref{thm:gt_mt_convergence}, and in \eqref{eq:cor_pe_transferability_bound} we group terms and substitute $C_{\bbPhi,1}=C_V+4e^{2B}C_\ccalT$ and $C_{\bbPhi,2}=e^BC_V+6e^{2B}C_\ccalT$. 

Substituting the eigenvector PE bounds into \eqref{eq:cor_pe_transferability_bound},
\begin{align}
    &\frac{1}{\mu(\ccalM)}\left\|\bbI_{N_1}\bm\Phi_{\bbG_1}(\ccalT,\bbX_1)
    -\bbI_{N_2}\bm\Phi_{\bbG_2}(\ccalT,\bbX_2)\right\|_{L^{1,2}(\ccalM)} \nonumber\\
    &\quad\leq C_{\bbPhi,1}\frac{\sqrt{K}\,\Omega_K}{\Gamma_K}\,(\varepsilon_{N_1}^2+\varepsilon_{N_2}^2)
    +C_{\bbPhi,2}(r_1+r_2), \label{eq:cor_lap_proof_final}
\end{align}
where in \eqref{eq:cor_lap_proof_final} we apply Theorem~\ref{thm:lap_pe_convergence} twice and group terms.
\end{proof}

\subsection{Proof of Corollary \protect\ref{cor:gnn_transferability}} \label{app:gnn_transferability_proof}

\begin{proof}

Similarl to the above. By Theorem~\ref{thm:gnn_mnn_convergence}, choosing probability $\delta/2$, GNN PEs are bounded as 
\begin{align}
    \|[\bbX_k]_i-f(x_i)\|_2
    &\leq \Delta_{\bbPsi(\bbG_k)}
    = C\varepsilon_{N_k}^2
    +\frac{\pi^2}{6}\sqrt{\frac{\log(2/\delta)}{N_k}}.
    \label{eq:cor_gnn_proof_pe_bound}
\end{align}

Substituting the GNN PE bounds into \eqref{eq:cor_pe_transferability_bound},
\begin{align}
    &\frac{1}{\mu(\ccalM)}\left\|\bbI_{N_1}\bm\Phi_{\bbG_1}(\ccalT,\bbX_1)
    -\bbI_{N_2}\bm\Phi_{\bbG_2}(\ccalT,\bbX_2)\right\|_{L^{1,2}(\ccalM)} \nonumber\\
    &\quad\leq C_{\bbPhi,1}\left[C(\varepsilon_{N_1}^2+\varepsilon_{N_2}^2)
    +\frac{\pi^2}{6}\left(\sqrt{\frac{\log(2/\delta)}{N_1}}
    +\sqrt{\frac{\log(2/\delta)}{N_2}}\right)\right] \nonumber\\
    &\qquad\quad+C_{\bbPhi,2}(r_1+r_2), \label{eq:cor_gnn_proof_final}
\end{align}
where in \eqref{eq:cor_gnn_proof_final} we apply ~\eqref{eq:cor_gnn_proof_pe_bound} twice with confidence parameter $\delta/2$ and group terms.
\end{proof}

\subsection{Sparse Graph Transformer}\label{app:sparse_gt}
One common computational tradeoff in Graph Transformers is to mask coefficients to attend to a neighborhood of nodes reachable in $k$-hops instead of computing all pairwise node attentions. Restricting the attention operation to the $k$-hop neighborhood results in the Sparse Graph Transformer (Sparse GT) given by
\begin{align}
    \dot{\bm\Phi}_\bbG(\ccalT,\bbX)(x_i) = \frac{\sum_{j\in\ccalN^{\leq k}(i)} \exp \{\langle \bbQ\bbx_{i},\bbK\bbx_{j} \rangle\} \bbV\bbx_{j} }{\sum_{j\in\ccalN^{\leq k}(i)}\exp \{\langle \bbQ\bbx_{i},\bbK\bbx_{j} \rangle\}}, \label{eqn:sgt_vec_attn}
\end{align}
where we denote the $k$-hop neighborhood of node $i$ as $\ccalN^{\leq k}(i) = \bigcup_{k'=0}^k \ccalN^{k'}(i)$, with $\ccalN^0(i) = \{i\}$ and $\ccalN^k(i) = \{j : \exists\, j' \in \ccalN^{k-1}(i),\; j \in \ccalN(j')\}$ for $k \geq 1$. Sparse GT reduces computational complexity from quadratic in $N$ to quadratic in the worst-case $k$-hop neighborhood cardinality.

\section{Experimental Details} \label{app:experimental_details}
This section contains an extended discussion of the experimental methodology for the transferability experiment in the main body.

\subsection{Transferability experiments on graph datasets}

\paragraph{Subsampling protocol.} The sampling procedure for induced graphs is the following: we random sample nodes uniformly. For any two points in the sampled set, if there exists an edge in the original graph, the induced subgraph includes that edge. Since the sampling is uniform, the label proportion is maintained in expectation. With respect to feature preprocessing, we use the features of each graph dataset as they are, with the exception of batch normalization at the input. 

\textbf{Datasets.} SNAP-Patents is a network for patents granted between 1963 and 1999 in the US. Each node is a patent, and a directed edge connects a patent to another patent that it cited. The prediction task is to classify each patent into one of five time intervals. ArXiv-year is a paper citation network on the arXiv papers. Each node represents a paper, and a directed edge connects a paper to another paper that it cited. The task is to predict the posting time of each paper, which is classified into one of five time intervals between 2013 and 2020. Both SNAP-Patents and arXiv-year are heterophilic graph datasets. OGBN-MAG is a heterogeneous network composed of a subset of the Microsoft Academic Graph (MAG), capturing the relationships among papers, authors, institutions and topics. The node classification task is to predict the venue of each paper. Reddit-Binary consists of 2000 graphs, each representing a subreddit community. The graph-level binary classification task is to identify each community as either a Q\&A community or discussion-based community using graph structures.

\paragraph{Model implementation details.} For both dense and sparse GTs, we implement multiheaded, scaled dot-product self-attention~\citep{vaswaniAttention2017} with feedforward layers. RPEARL's GNNs are implemented as in Equation~\eqref{eqn:gnn_matrix_form_pe} using TAGConv filters~\citep{duTopologyAdaptiveGraph2018}.\todo{expand and update}

\begin{table}[h]
\centering
\caption{Datasets used for transferability experiments. Edges are directed edges excluding self-loops; for OGBN-MAG we use the paper--paper citation subgraph. Max train/test nodes are the sizes of the full training and test graphs ($\alpha = 1$), with 10\%/45\% of nodes held out for validation/test (20\%/40\% for Squirrel).}
\resizebox{0.95\textwidth}{!}{
\begin{tabular}{rcccccc}
\toprule
\textbf{Dataset} & Nodes & Edges & Max Train/Test Nodes & Classes & Graphs & Feature Dim \\
\midrule
Squirrel      & 5,201          & 216,933        & 2,081         & 5   & 1    & 2,089 \\
ArXiv-year    & 169,343        & 1,166,243      & 76,205       & 5   & 1    & 128 \\
OGBN-MAG      & 736,389        & 5,416,271      & 331,376     & 349 & 1    & 128 \\
SNAP-Patents  & 2,923,922      & 13,975,788     & 1,315,765 & 5   & 1    & 269 \\
Reddit-Binary & avg.\ 429.63   & avg.\ 497.75   & /                     & 2   & 2000 & 1$^{*}$ \\
\bottomrule
\end{tabular}
}
\\[2pt]
{\footnotesize $^{*}$Reddit-Binary has no node features; a constant feature is used.}
\label{tab:datasets}
\end{table}

\paragraph{Dataset preparation.} Each dataset is split into train/val/test fractions of 45\%-10\%-45\% respectively. For datasets that have pre-established train/test masks, we discard the masks in favor of our partition proportions. The training and testing partitions are the sources for the graph subsampling procedure explained in Section \ref{sec:experiments}, where the largest testing graph $\alpha=1.0$ corresponds to using the full 45\% training fraction and 45\% of the testing fraction. 

\paragraph{Training procedure and transferability evaluation.} For each model architecture, we train multiple models with graphs of increasing sizes and evaluate each model on testing graphs of different sizes. For single-graph datasets, the training graphs and testing graphs are constructed by subsampling a fraction of nodes from the training split and the testing split respectively. For multi-graph datasets, we construct the training graphs and testing graphs by subsampling the nodes in each training graph and each testing graph respectively by a specific fraction. We do not create graph batches, but rather train with the full subsampled graph.

\paragraph{Hyperparameters.} For each dataset, we tuned hyperparameters by training each model on the full training graph. Then, we held the selected hyperparameters fixed across all smaller training graph sizes. In general, we tuned learning rate over the interval $[1 \times 10^{-5}, 1 \times 10^{-2}]$, dropout and attention dropout over $[0, 0.5]$, the feedforward hidden dimension in $\{128, 256, 512\}$, the number of transformer heads in $\{4, 8\}$, the number of GNN/Transformer layers in $\{3, 4, 5, 6, 7\}$, the number of GNN/Sparse GT hops in $\{2, 3, 4\}$, and the Expander degree for Exphormer in $\{3, 4\}$. The full set of hyperparameters used for each model and dataset is available in Table \ref{tab:hyperparameters}. 

\paragraph{Computational resources}
The experiments reported in this paper were executed on NVIDIA DGX B200 instances, using a single GPU per experiment (180GB of HBM3e Memory per process). Preliminary hyperparameter tuning was performed on a machine with two NVIDIA GeForce RTX 3090 GPUs, with 24GB of GPU memory each.
\begin{table}[htbp]
\centering
\caption{Hyperparameters for each architecture and dataset. PE: positional encoding; for RPEARL, ``Random'' feeds random node signals and ``Data'' the node features into the PE network.}
\resizebox{\textwidth}{!}{
\begin{tabular}{lccccc}
\toprule
\textbf{Model/Hyperparameters} & Squirrel & ArXiv-year & OGBN-MAG & SNAP-Patents & Reddit-Binary \\
\midrule
\textbf{GNN} & & & & & \\
\quad Learning rate & $1 \times 10^{-3}$ & $1 \times 10^{-2}$ & $1 \times 10^{-2}$ & $1 \times 10^{-3}$ & $1 \times 10^{-3}$ \\
\quad Dropout & 0.5 & 0.5 & 0.25 & 0.25 & 0 \\
\quad Hidden channels & 256 & 256 & 128 & 256 & 512 \\
\quad Number of hops & 3 & 3 & 3 & 3 & 3 \\
\quad Number of layers & 2 & 3 & 7 & 7 & 2 \\
\quad Max epochs & 300 & 300 & 700 & 800 & 350 \\
\midrule
\textbf{Dense GT + Eigenvector PE} & & & & & \\
\quad Learning rate & $1 \times 10^{-3}$ & $5 \times 10^{-4}$ & $5 \times 10^{-4}$ & $5 \times 10^{-4}$ & $5 \times 10^{-4}$ \\
\quad Dropout & 0.25 & 0.05 & 0.25 & 0.05 & 0.1 \\
\quad Attention dropout & 0.5 & 0.15 & 0.15 & 0.15 & 0.1 \\
\quad Heads & 8 & 4 & 4 & 4 & 8 \\
\quad Dim model & 64 & 256 & 128 & 256 & 128 \\
\quad Dim feedforward & 128 & 128 & 384 & 128 & 512 \\
\quad Number of layers & 1 & 3 & 3 & 3 & 3 \\
\quad Eigenvectors ($k$) & 16 & 8 & 8 & 8 & 8 \\
\quad Max epochs & 300 & 300 & 700 & 300 & 350 \\
\midrule
\textbf{Dense GT + RPEARL} & & & & & \\
\quad Learning rate & $1 \times 10^{-3}$ & $5 \times 10^{-4}$ & $5 \times 10^{-4}$ & $5 \times 10^{-4}$ & $5 \times 10^{-4}$ \\
\quad Dropout & 0.25 & 0.05 & 0.05 & 0.05 & 0.1 \\
\quad Attention dropout & 0.5 & 0.15 & 0.15 & 0.15 & 0.1 \\
\quad Heads & 8 & 4 & 4 & 4 & 8 \\
\quad Dim model & 64 & 256 & 128 & 256 & 128 \\
\quad Dim feedforward & 128 & 128 & 384 & 128 & 512 \\
\quad Number of layers & 1 & 3 & 3 & 3 & 3 \\
\quad PE input & Random & Random & Data & Data & Random \\
\quad PE hidden channels & 128 & 128 & 128 & 128 & 512 \\
\quad PE number of layers & 3 & 8 & 8 & 8 & 4 \\
\quad PE random samples & 100 & 30 & -- & -- & 30 \\
\quad Max epochs & 300 & 300 & 700 & 300 & 350 \\
\midrule
\textbf{Sparse GT + Eigenvector PE} & & & & & \\
\quad Learning rate & $1 \times 10^{-3}$ & $5 \times 10^{-4}$ & $5 \times 10^{-4}$ & $5 \times 10^{-4}$ & $1 \times 10^{-5}$ \\
\quad Dropout & 0.01 & 0.01 & 0.5 & 0.01 & 0.01 \\
\quad Attention dropout & 0.5 & 0.01 & 0.05 & 0.01 & 0.01 \\
\quad Heads & 8 & 8 & 4 & 8 & 8 \\
\quad Dim model & 64 & 256 & 128 & 256 & 128 \\
\quad Number of hops & 1 & 2 & 2 & 2 & 2 \\
\quad Number of layers & 1 & 3 & 3 & 8 & 3 \\
\quad Eigenvectors ($k$) & 4 & 8 & 8 & 8 & 8 \\
\quad Max epochs & 300 & 300 & 4000 & 1500 & 350 \\
\midrule
\textbf{Sparse GT + RPEARL} & & & & & \\
\quad Learning rate & $1 \times 10^{-3}$ & $5 \times 10^{-4}$ & $5 \times 10^{-3}$ & $5 \times 10^{-4}$ & $1 \times 10^{-5}$ \\
\quad Dropout & 0.01 & 0.01 & 0.3 & 0.01 & 0.01 \\
\quad Attention dropout & 0.5 & 0.01 & 0.25 & 0.01 & 0.01 \\
\quad Heads & 8 & 8 & 4 & 8 & 8 \\
\quad Dim model & 64 & 256 & 128 & 256 & 128 \\
\quad Number of hops & 1 & 2 & 2 & 2 & 2 \\
\quad Number of layers & 1 & 3 & 4 & 8 & 3 \\
\quad PE hidden channels & 128 & 128 & 128 & 128 & 128 \\
\quad PE number of layers & 3 & 8 & 1 & 8 & 8 \\
\quad PE random samples & 30 & 30 & 5 & 30 & 30 \\
\quad Max epochs & 300 & 300 & 700 & 1500 & 350 \\
\bottomrule
\end{tabular}
}
\label{tab:hyperparameters}
\end{table}

\section{Additional Experimental Results} \label{app:more_results}

\subsection{Additional datasets and architectures} \label{app:additional_results}
In this section, we provide additional results, including transferability analysis for other reference architectures, as well as the full heatmaps for every model-dataset combination.

\paragraph{Additional models.} In addition to Dense GT and Sparse GT,  we consider MLP and GNN~\citep{duTopologyAdaptiveGraph2018} baselines, GraphGPS~\citep{rampavsek2022recipe}, a GNN-transformer hybrid with Laplacian eigenvector PEs, and Exphormer~\citep{shirzad2023exphormer}, a sparse variant of Graph Transformer that computes attention coefficients for (i) one-hop neighbors, (ii) random edges from an expander graph, (iii) $N$ additional attention coefficients between each node and a virtual global node. Crucially, Exphormer has no positional encodings beyond its masking.

\paragraph{Transferability and comparison between models.} Figure~\ref{fig:all_transferability} contains transferability plots for all architectures. We observe transferability patterns for Dense GT, Sparse GT and GNN in most datasets. Exphormer shows transferable behavior in OGBN-MAG and SNAP-Patents.  In all datasets, GT is either comparable (OGBN-MAG, Reddit-Binary) or significantly superior (SNAP-Patents, ArXiv-year) to GNN. This performance gap is consistent with previous work's observations of the success of global attention in heterophilic datasets 
\citep{dwivedi2020generalization}. The trend of Exphormer in ArXiv-year also shows a more pronounced monotonic increase, possibly due to the random sampling procedure being more beneficial on larger graphs. In OGBN-MAG and Reddit-Binary, GNN and GTs show comparable performance and transferability patterns. In the cases of SNAP-Patents and OGBN-MAG, Exphormer presents a monotonically increasing performance as training graph size increases.

\begin{figure}[htbp]
    \centering
    \includegraphics[width=0.95\linewidth]{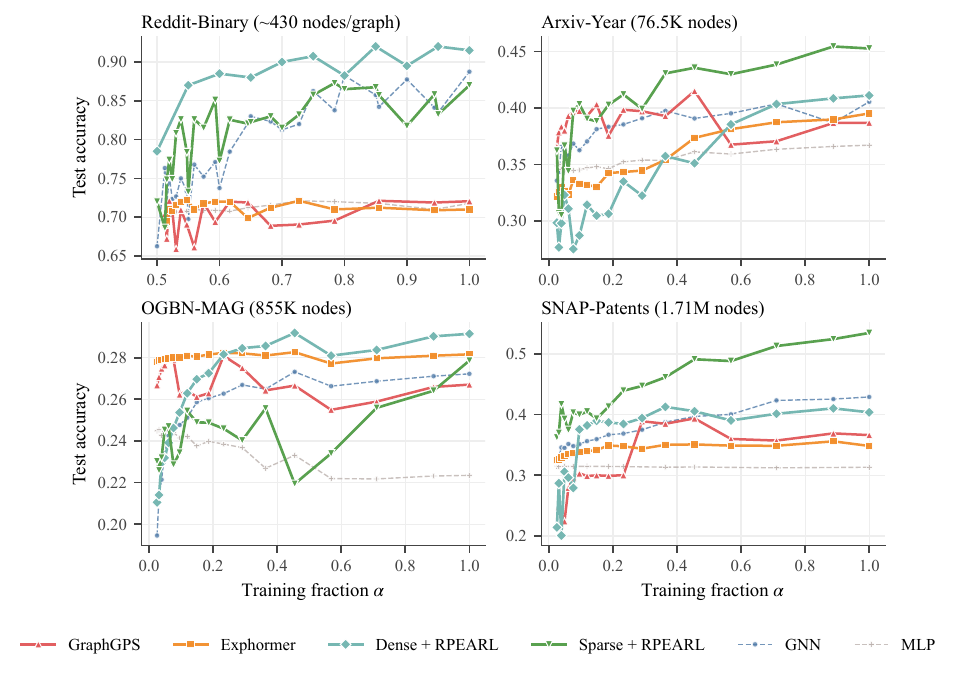}
    \caption{Transferability plots for all datasets and architectures. For each dataset, the $x$ axis represents the train graph sizes as a proportion of the largest graph $(\alpha)$, and the $y$ axis is the test accuracy at the full-sized graph. The titles show dataset name and largest graph size.}
    \label{fig:all_transferability}
\end{figure}

\begin{figure}[htbp]
    \centering
    \includegraphics[width=0.95\linewidth]{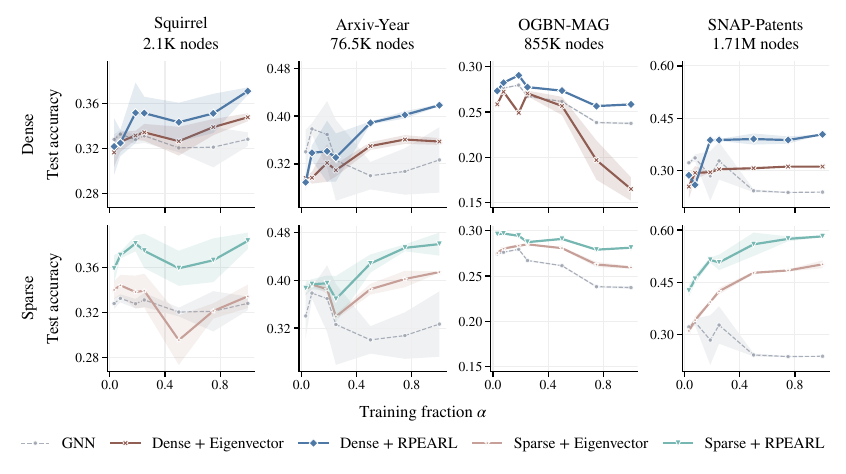}
    \caption{ERROR BARS}
    \label{fig:transferability_error_bars}
\end{figure}

\todo{decide if add the heatmaps}

\paragraph{Transference to smaller test fractions.} Transferability also implies that model performance transfers to graph sizes smaller than the training graph size. Figure~\ref{fig:heatmaps_sparse_dense_and_sparse_gt} illustrates this effect with a heatmap of test accuracy at various combinations of training (row) and testing (column) graph sizes. Both sparse and dense GT exhibit competitive accuracy at various settings.

\begin{figure}
    \centering
    \includegraphics[width=\textwidth]{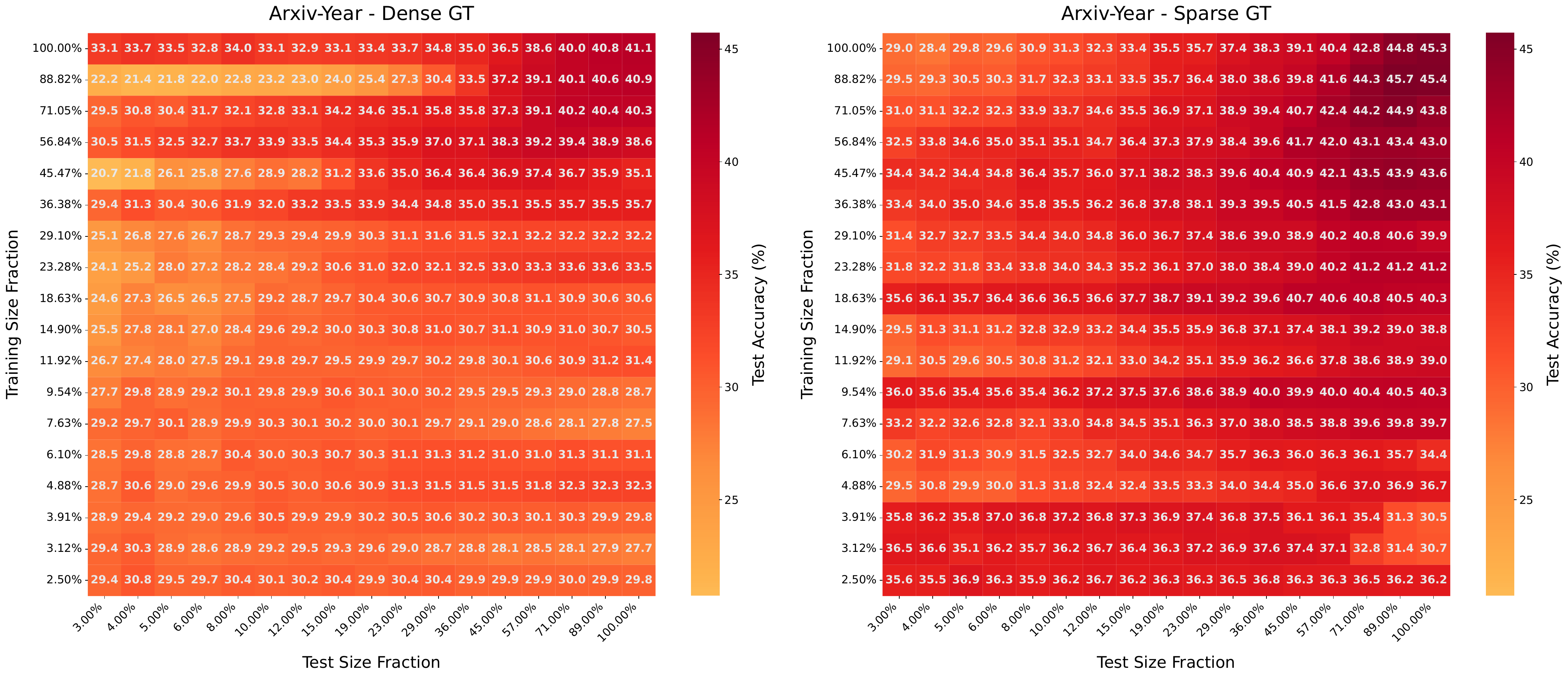}
    \caption{Transferability heatmaps on ArXiv-year. The $x$ axis corresponds to train graph sizes as a proportion of the largest graph, and the $y$ axis to test graph size fractions. The color corresponds to the test accuracy at each setting.}
    \label{fig:heatmaps_sparse_dense_and_sparse_gt}
\end{figure}

\subsection{Ablation analysis and effect of masking} \label{app:ablation} \label{sec:experiments_ablation}

\paragraph{RPEARL + Masking achieves the best transference on SNAP-Patents.} We perform an ablation analysis of Sparse GT's structural components on the SNAP-Patents dataset to gain further insight into the effect of each mechanism on its transferability. For this, we train each architecture on a $30\%$ downsample (513K nodes) of the largest training set, and evaluate on a 1.71M node test set. In addition to ablating RPEARL and Mask, we also consider random edges (RE) using random expander graphs of degree 3~\citep{shirzad2023exphormer}.  On Table~\ref{tab:ablation}, we can observe that RPEARL encodings enhance the baseline GT performance by 10.53\%. Adding the sparse encoding mask improves performance by 26.68\%.  Combining both encodings yields a performance gain of 58.63\%, reaching the best performance across all settings. Finally, random edges do not seem to improve performance.  Our findings coincide with observations from previous works that this attention mask may be a useful inductive bias~\citep{linUnifiedGTUniversalFramework2024}. 

\paragraph{$k$-hop parameter ablation.}  The $k$-hop neighborhood mask of Sparse GT trades off computational complexity and global information, resulting in up to 100x training speedups relative to Dense GT, while retaining comparable performance (see Table~\ref{tab:runtimes}). Beyond this tradeoff, we also observed improved test performance in the previous ablation. In our experiments, we find all values of $k$ to be similarly transferable, with the best performance in $k\in\{3,4\}$. In Figure~\ref{fig:transferability_k_hop_ablation}, we observe little sensitivity to the value of $k$. In this particular case, the optimal value is at $k=4$.

\begin{figure}[b]
    \centering
    \includegraphics[width=0.6\textwidth]{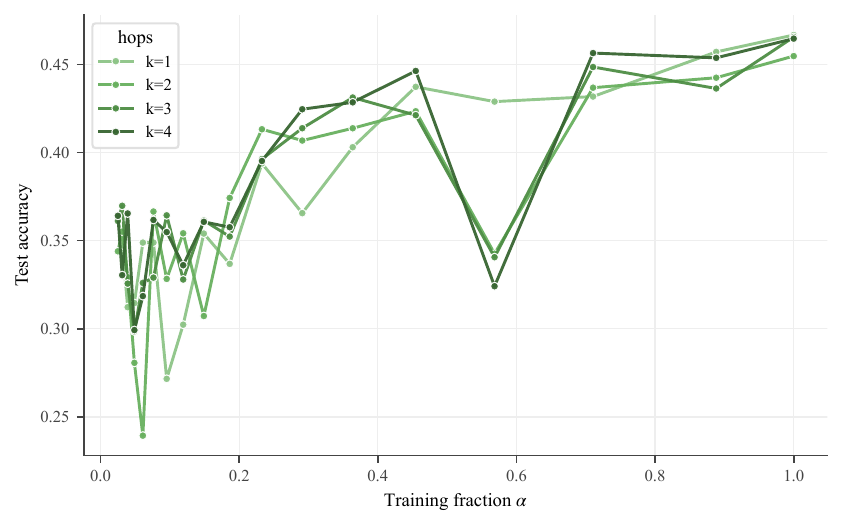}
    \caption{Transferability plot for Sparse GT on ArXiv-year for various values of the neighborhood parameter $k$.}
    \label{fig:transferability_k_hop_ablation}
\end{figure}

\begin{table}
    \centering
    
    \caption{Test accuracy of GT trained with $\alpha=0.3$, with different components in SNAP-Patents (1.71M nodes).}
    \label{tab:ablation}
     \begin{tabular}{l c c}
    \toprule
    \textbf{Architecture} & \textbf{Accuracy} & \textbf{\% vs. GT} \\
    \midrule
    GT (no PE) & 31.33 & -- \\
    GT + RPEARL & 34.63 & +10.53\% \\
    GT + Mask & 39.69 & +26.68\% \\
    GT + Mask + RPEARL & \textbf{49.70} & \textbf{+58.63\%} \\
    GT + Mask + RE & 31.01 & -1.02\% \\
    GT + Mask + RPEARL + RE & 46.55 & +48.58\% \\
    \bottomrule
    \end{tabular}

\end{table}

\subsection{Runtime analysis} \label{app:runtime}
Our transferability results imply that it is possible to train with small graphs and attain competitive test performance on larger graphs, resulting in shorter training times. In addition to this, Sparse GT reduces  the computational complexity from quadratic in the number of nodes to quadratic in the $k$-hop neighborhood cardinality. Table~\ref{tab:runtimes} shows the training times of Dense and Sparse GT for multiple fractions. We observe that Sparse GT is one or two orders of magnitude faster than Dense GT in all settings. It is worth highlighting that Sparse GT is both faster and achieves superior performance than Dense GT in ArXiv-year, and SNAP-Patents.

\begin{table}[]
    \centering
     \caption{Runtimes in minutes for three subsample sizes of Sparse and Dense GT}
    \begin{tabular}{llrrr}
        \toprule
        \textbf{Dataset} & \textbf{Model} & \textbf{6.10\%} & \textbf{56.84\%} & \textbf{100.00\%} \\
        \midrule
        ArXiv-year &Dense GT & 51.6 & 10.0 & 19.8 \\
        ArXiv-year &Sparse GT & 8.8 & 1.4 & 2.0 \\
        OGBN-MAG & Dense GT & 41.8 & 103.3 & 366.6 \\
        OGBN-MAG & Sparse GT & 22.4 & 7.4 & 15.0 \\
        SNAP-Patents & Dense GT & 223.3 & 945.7 & 2763.6 \\
        SNAP-Patents & Sparse GT & 28.3 & 13.5 & 23.2 \\
        \bottomrule
    \end{tabular}
    \label{tab:runtimes}
\end{table}

\subsection{Alternative Positional Encodings} \label{app:other_pes}

Our experiments from Section~\ref{sec:experiments} focused on analyzing the transferability of a variety of combinations of GT backbone and PE architecture. In this section we provide additional results on the transferability of multiple PE variants for a fixed GT backbone. In Figure~\ref{fig:rpearl_ablation}, we examine the transferability with different GNN architectures, all using RPEARL-style random features as described in Appendix~\ref{app:rpearl}. We observe similar transferability in all cases, with the exception of GIN, where we observe a slight degradation across training sizes. 

We also compare the effect of different families of PE in Figure~\ref{fig:posenc_comparison}. We observe similar performance across training fractions for RWPE and Laplacian-based PEs. Relative to these, RPEARL PEs exhibit improved performance across all training fractions.

\begin{figure}
    \centering
    \begin{subfigure}{0.49\linewidth}
        \centering
        \includegraphics[width=\linewidth]{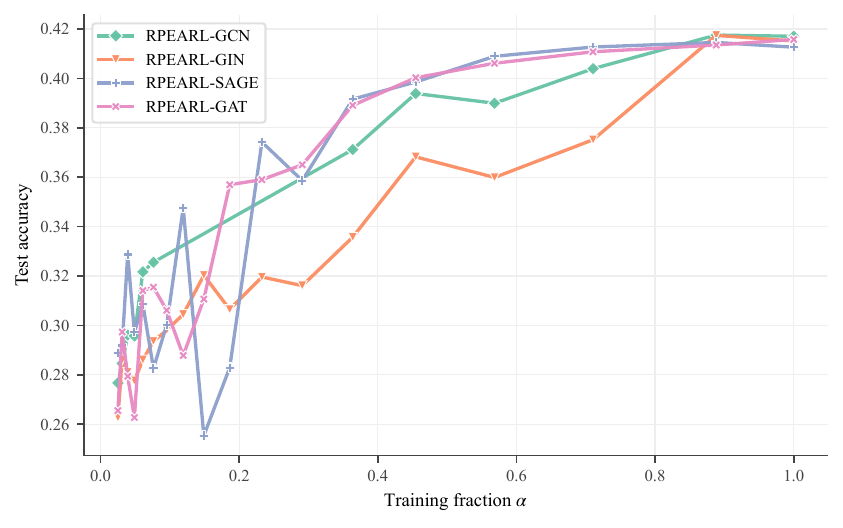}
        \caption{Transferability of RPEARL with different GNNs for the positional encoding backbones. Dataset: ArXiv-year.}
        \label{fig:rpearl_ablation}
    \end{subfigure}
    \hfill
    \begin{subfigure}{0.49\linewidth}
        \centering
        \includegraphics[width=\linewidth]{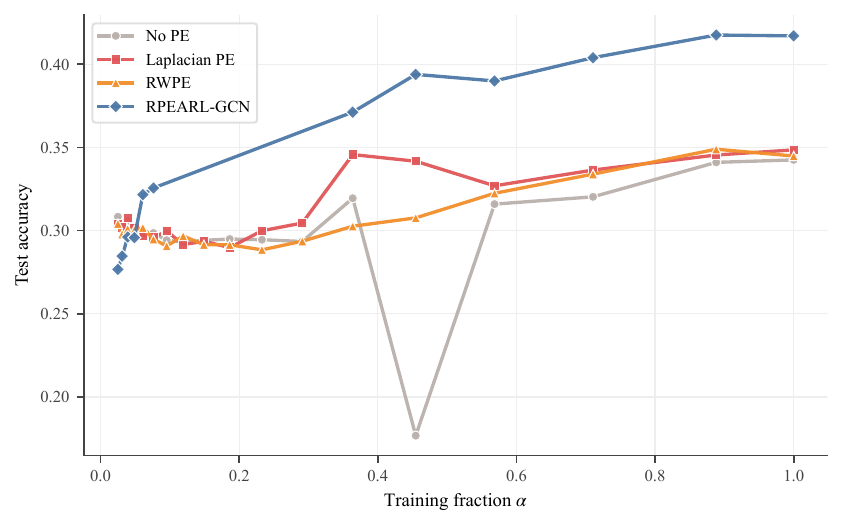}
        \caption{Transferability of different positional encodings with the same GT backbone (dense). Dataset: ArXiv-year.}
        \label{fig:posenc_comparison}
    \end{subfigure}
    \label{fig:posenc_combined}
\end{figure}

\subsection{Terrain graphs} \label{app:terrain_graphs}
In Section~\ref{sec:experiments_terrain_graph} we presented an application of our theoretical results to the problem of distance estimation over terrain graphs. Terrain graphs are point cloud approximations of a physical manifold. Some time-critical applications require low latency Shortest Path Distance (SPD) estimations between two points on a terrain, which is challenging in high-resolution point clouds with millions of nodes. It is also required that SPD estimations are computed accurately over various resolutions. This has been tackled in deep learning as a \textit{metric learning} problem, where the task is to learn a latent space using a GNN in which the embedding distance between two points approximates the true SPD between them.
 
Sparse GT offers an efficient, scalable, and expressive architecture that can be applied to the metric learning problem in terrain graphs. Furthermore, the transferability guarantees of Corollary~\ref{cor:gt_transferability} imply that we need not train with the highest possible resolution of terrain graph in order to achieve competitive performance. Here we provide implementation details and an additional result on a second terrain graph dataset.

\todo{Add TG Hparams}

\textbf{Metric learning.} The setup of ~\citep{chenDecoupledNeuroGFShortest2025} is to use imitation learning to learn a latent space that resembles true shortest-path distance between points. For two points $x_i,x_j$ in the terrain graph, we train to minimize $\| L_1(\phi(x_i)-\phi(x_j)) - SPD(x_i,x_j)\|_2$, where $\phi$ is an embedding function, in our case either GNN or SGT. On~\citep{chenDecoupledNeuroGFShortest2025}, a subsequent stage consists of freezing the embedding function's parameters and finetuning an additional MLP on a larger set of sample points. In our experiments, we omit this second stage. During training, we bias the point pair sampling to increase the frequency of difficult paths (points with higher altitude variance and shortest distance).

\paragraph{Datasets.} The terrain graphs we use are Digital Elevation Model (DEM) datasets. The results of Section~\ref{sec:experiments_terrain_graph} used a DEM of the Troms region of Norway~\citep{kartverket_dem_2025}. The Norway dataset is available in a $2000\times2000$ (4M nodes) resolution. For our experiments we downsample so that the largest grid graph is $250\times250$. 

\paragraph{Graph and train/test data constructions.} To test transferability, we generate graphs of evenly spaced points based on the high resolution graph. The downsampling parameter $r$ controls the space between generated points, i.e. $r=1$ is the full resolution, $r=2$ takes every second point (resulting in a $1000\times1000$ graph), etc. The edges of the graph are taken to be the 8 nearest neighbors. For training datasets, we subsample 500 sources and 100 targets per source, for a total of $50,000$ pairs. \todo{explain the training dataset sampling weight.} For testing, a total of $10,000$ source points are selected based on maximum height, and the SPD is computed between these points and every other point in the graph.


\end{document}